\documentclass[nohyperref]{article}

\usepackage{microtype}
\usepackage{graphicx}
\usepackage{booktabs} 

\usepackage{hyperref}

\usepackage{caption}
\usepackage{multirow}
\usepackage{graphicx}
\usepackage{subcaption}
\usepackage{caption}
\usepackage{placeins}
\usepackage{booktabs}
\usepackage{color}
\usepackage{lipsum}

\newcommand{\RNum}[1]{\uppercase\expandafter{\romannumeral #1\relax}}
\newcommand{\bm}[1]{\boldsymbol{#1}}

\newcommand{\Black}[1]{\textcolor[rgb]{0.00,0.00,0.00}{#1}}

\newcommand{\revisexin}[1]{\Black{#1}}

\newcommand{\revisejy}[1]{\Black{#1}}

\usepackage[accepted]{icml2022}

\usepackage{amsmath}
\usepackage{amssymb}
\usepackage{mathtools}
\usepackage{amsthm}
\usepackage{enumitem}

\usepackage[capitalize,noabbrev]{cleveref}

\theoremstyle{plain}
\newtheorem{theorem}{Theorem}[section]
\newtheorem{proposition}[theorem]{Proposition}
\newtheorem{lemma}[theorem]{Lemma}
\newtheorem{corollary}[theorem]{Corollary}
\theoremstyle{definition}
\newtheorem{definition}[theorem]{Definition}

\theoremstyle{remark}

\usepackage[textsize=tiny]{todonotes}

\icmltitlerunning{Boosting Graph Structure Learning with Dummy Nodes}

\begin{document}

\twocolumn[
\icmltitle{Boosting Graph Structure Learning with Dummy Nodes}




\begin{icmlauthorlist}
\icmlauthor{Xin Liu}{ust}
\icmlauthor{Jiayang Cheng}{ust}
\icmlauthor{Yangqiu Song}{ust}
\icmlauthor{Xin Jiang}{hw}
\end{icmlauthorlist}

\icmlaffiliation{ust}{Department of Computer Science and Engineering, Hong Kong University of Science and Technology, Hong Kong SAR, China}
\icmlaffiliation{hw}{Huawei Noah's Ark Lab, Hong Kong SAR, China}

\icmlcorrespondingauthor{Xin Liu}{xliucr@cse.ust.hk}

\icmlkeywords{graph structure learning, graph kernel, graph neural network, graph isomorphism}

\vskip 0.3in
]



\printAffiliationsAndNotice{This work was done when Xin Liu was an intern at Huawei Noah’s Ark Lab.} 

\begin{abstract}
With the development of graph kernels and graph representation learning, many superior methods have been proposed to handle scalability and oversmoothing issues on graph structure learning.
However, most of those strategies are designed based on practical experience rather than theoretical analysis.
In this paper, we use a particular dummy node connecting to all existing vertices without affecting original vertex and edge properties.
We further prove that such the dummy node can help build 
an efficient monomorphic edge-to-vertex transform and an epimorphic inverse to recover the original graph back.
It also indicates that adding dummy nodes can preserve local and global structures for better graph representation learning.
We extend graph kernels and graph neural networks with dummy nodes and conduct experiments on graph classification and subgraph isomorphism matching tasks.
Empirical results demonstrate that taking graphs with dummy nodes as input significantly boosts graph structure learning, and using their edge-to-vertex graphs can also achieve similar results.
\revisexin{We also discuss the gain of expressive power from the dummy in neural networks.}
\end{abstract}

\section{Introduction}
\label{sec:introduction}
Graph structures have been widely used in modeling the interactions and connections in complex systems, such as biological networks, chemical molecules, and social networks.
In fact, these data contain rich information in their graph structure beyond vertex and edge attributes.
For example, atoms are held together by covalent bonds, and different molecular compounds (usually named isomers) with similar atoms may have distinct properties due to different structures.
Thus, there has been a surge of interest in graph similarity, graph comparison, and subgraph matching.
In recent years, numerous approaches have been proposed in machine learning and deep learning, among which algorithms based on graph kernels (GKs)~\cite{borgwardt2005protein,shervashidze2011weisfeiler,morris2020weisfeiler} and graph neural networks (GNNs)~\cite{kipf2017semi,vashishth2020composition} are most notable.
GKs tackle the graph comparison by exploring and capturing the semantics inherent in graph-structured data.
The main idea behind graph kernels is that graphs with similar properties are highly likely to have similar distributions of substructures.
\revisexin{However, most GKs focus on vertex-centric substructures while ignoring the edge similarities.}
Instead, inductive GNNs automatically extract higher-order information of graphs, sometimes leading to more powerful features compared to hand-crafted features used by GKs~\cite{xu2019how}.
Nevertheless, there are also disadvantages of GKs and GNNs.
The runtime complexity of GKs when considering
subgraphs of size up to $k \geq 2$ is usually $\Omega(k \cdot |\mathcal{V}|^{2})$ at least and $\mathcal{O}(k \cdot |\mathcal{V}|^{k + 1})$ at worst, where $|\mathcal{V}|$ is the number of vertices in the smaller graph for comparison~\cite{kriege2020asurvey}.
The objective of GNNs is highly non-convex, requiring careful hyper-parameter tuning to stabilize the training procedure and avoid oversmoothing.
Besides, message passing in GNNs still faces the limitation of the expressive power \revisexin{and the information vanish upon 0-outdegree vertices in deeper networks}.

To address the aforementioned problems, many strategies have been designed.
\citeauthor{morris2020weisfeiler}~(\citeyear{morris2020weisfeiler}) proposed the local variants of Weisfeiler-Lehman Subtree Kernels ($k$-WL) to vastly reduce the computation time without performance decline.
$k$-IGNs~\cite{maron2019on}, $k$-GNNs~\cite{morris2019weisfeiler}, and  EASN~\cite{bevilacqua2021equivariant} involve high-order tensors in representing high-order substructures with the expressive power as $k$-WL.
Some kernels quantify the similarity based on random walks to reduce the complexity~\cite{zhang2018retgk}.
\citeauthor{li2018deeper}~(\citeyear{li2018deeper}) and \citeauthor{rong2019the}~(\citeyear{rong2019the}) addressed the GNNs' oversmoothing by randomly removing edges from graphs to make GNNs robust to various structures.
\revisexin{On the other hand, adding reversed edges in heterogeneous graphs is popular in practice~\cite{vashishth2020composition,hu2020heterogeneous}.}
Some strategies are based on \revisexin{experimental experience} rather than theoretical analysis.

To derive a theoretically guaranteed (sub)graph structure modeling that is consistent in original GKs computation and can improve GNNs,
we use a particular dummy node and connect it with all existing vertices without affecting original vertex and edge properties.
We start from the edge-to-vertex transform to theoretically analyze the role of the dummy in  structure preserving.
It turns out that an efficient monomorphic transform $L_{\Phi}$ to convert edges to vertices and an epimorphic inverse $L_{\Phi}^{-1}$ to recover the original graph back make the edge-to-vertex transform lossless.
It is interesting to observe the transformed graph also contains another dummy node.
We extend vertex-centric GKs and GNNs with dummy nodes to boost graph structure learning\revisexin{, with a linear computation cost to the number of edges.}

Our main contributions are highlighted as follows:
\begin{enumerate}
    \item We prove that adding a special dummy node with links to current existing vertices can help build an efficient and lossless edge-to-vertex transform\revisexin{, which also indicates that the edge information can be well preserved during learning.}
    \item We utilize dummy nodes and edges to extend state-of-the-art machine learning and deep learning models to improve their abilities to capture (sub)graph structures.
    \item Extensive experiments are conducted on graph classification and subgraph isomorphism counting and matching, and empirical results reveal the success of learning with graphs with dummy nodes.
\end{enumerate}
Code is publicly released at \url{https://github.com/HKUST-KnowComp/DummyNode4GraphLearning}.
\section{Related Work}
Before the deep learning era, graph kernels (GKs) dominated supervised graph structure learning through mapping graphs to Hilbert space and computing Gram matrices.
Different kernel functions focus on specific structural properties of graphs.
Shortest-path Kernel~\cite{borgward2005shortest} decomposes graphs into shortest paths and compares graphs according to their shortest paths, such as path lengths and endpoint labels.
Instead, Graphlet Kernels~\cite{shervashidze2009efficient} compute the distribution of small subgraphs under the assumption that graphs with similar graphlet distributions are highly likely to be similar.
Another important kernel family considers subtrees.
One state-of-the-art kernel is Weisfeiler-Lehman Subtree Kernel ($k$-WL)~\cite{shervashidze2011weisfeiler}, and some higher-order variants~\cite{morris2017glocalized} and local variants~\cite{morris2020weisfeiler} further strengthen the expressive power.
However, graph kernels are limited by non-inductive learning and the super quadratic time complexity to the training data size.

With the rapid development of heterogeneous computing, neural networks have attracted attention recently.
Researchers have successfully used relational inductive biases within deep learning architectures to build graph neural networks (GNNs).
End-to-end learning relieves the burden of feature engineering and makes the structure learning sophisticated and flexible~\cite{gilmer2017neural,battaglia2018relational}.
Ideas behind kernels are still referable for the design of GNNs~\cite{morris2019weisfeiler}.
$k$-GNNs~\cite{morris2019weisfeiler} and EASN~\cite{bevilacqua2021equivariant} align the $k$-WL hierarchy to enhance the expressive power of GNNs.
However, this involves high-order tensor computations.
How to efficiently boost the graph structure learning is also one of the research frontiers.
Ranking neighbors based on attention scores enables explicit weights for information aggregation~\cite{yao2021node2seq}.
Besides, it also boosts learning to consider multihop neighbors~\cite{zhu2020beyond,teru2020inductive}.
On the other hand, dynamic high-order neighbor selection~\cite{yang2021graph} and dynamic pointer links~\cite{velickovic2020pointer} illustrate the power of data-driven manipulations.
Moreover, differentiable pooling yields consistent and significant performance improvement for end-to-end hierarchical graph representation learning~\cite{ying2018hierarchical,zhang2019hierarchical}.
\citeauthor{li2018deeper}~(\citeyear{li2018deeper}) and \citeauthor{rong2019the}~(\citeyear{rong2019the}) addressed the GNNs' oversmoothing by randomly removing edges from graphs to make GNNs robust to various structures.
However, most of these strategies are based on practical experience rather than theoretical analysis.

\revisexin{Some literature suggests utilizing ``dummy'' super-nodes to explicitly learn subgraphs~\cite{scarselli2009the,hamilton2017representation}, but the node is served as a special readout to conduct the representation of the target subgraph. On the contrary, we directly add a dummy node as a part of the target graph to learn representations and capture similarities.}
\section{Lossless Edge-to-vertex Transforms}

\subsection{Preliminary}
Let $\mathcal{G} = (\mathcal{V}_\mathcal{G}, \mathcal{E}_\mathcal{G}, \mathcal{X}_\mathcal{G}, \mathcal{Y}_\mathcal{G})$ be a \textit{directed connected heterogeneous multigraph} with a vertex set $\mathcal{V}_\mathcal{G}$, an edge set $\mathcal{E}_\mathcal{G} \subseteq \mathcal{V}_\mathcal{G} \times \mathcal{V}_\mathcal{G}$,
a label function $\mathcal{X}_\mathcal{G}$ that maps a vertex to a set of \textit{vertex labels}, and a label function $\mathcal{Y}_\mathcal{G}$ that maps an edge to a set of \textit{edge labels}.
Under this definition, multiple edges with the same source and the same target can be merged by extending $\mathcal{Y}$, making $\mathcal{G}$ without multiedges for clarity.
To simplify the statement, we also let $\mathcal{Y}_\mathcal{G}((u, v)) = \phi$ if $u,v \in \mathcal{V}_\mathcal{G}$ but $(u, v) \not \in \mathcal{E}_\mathcal{G}$.
We use $d_v^-$ and $d_v^+$ to denote the indegree and outdegree of vertex $v$.

\begin{figure*}[!ht]
    \centering
    \begin{subfigure}{.23\textwidth}
        \centering
        \includegraphics[width=4.0cm]{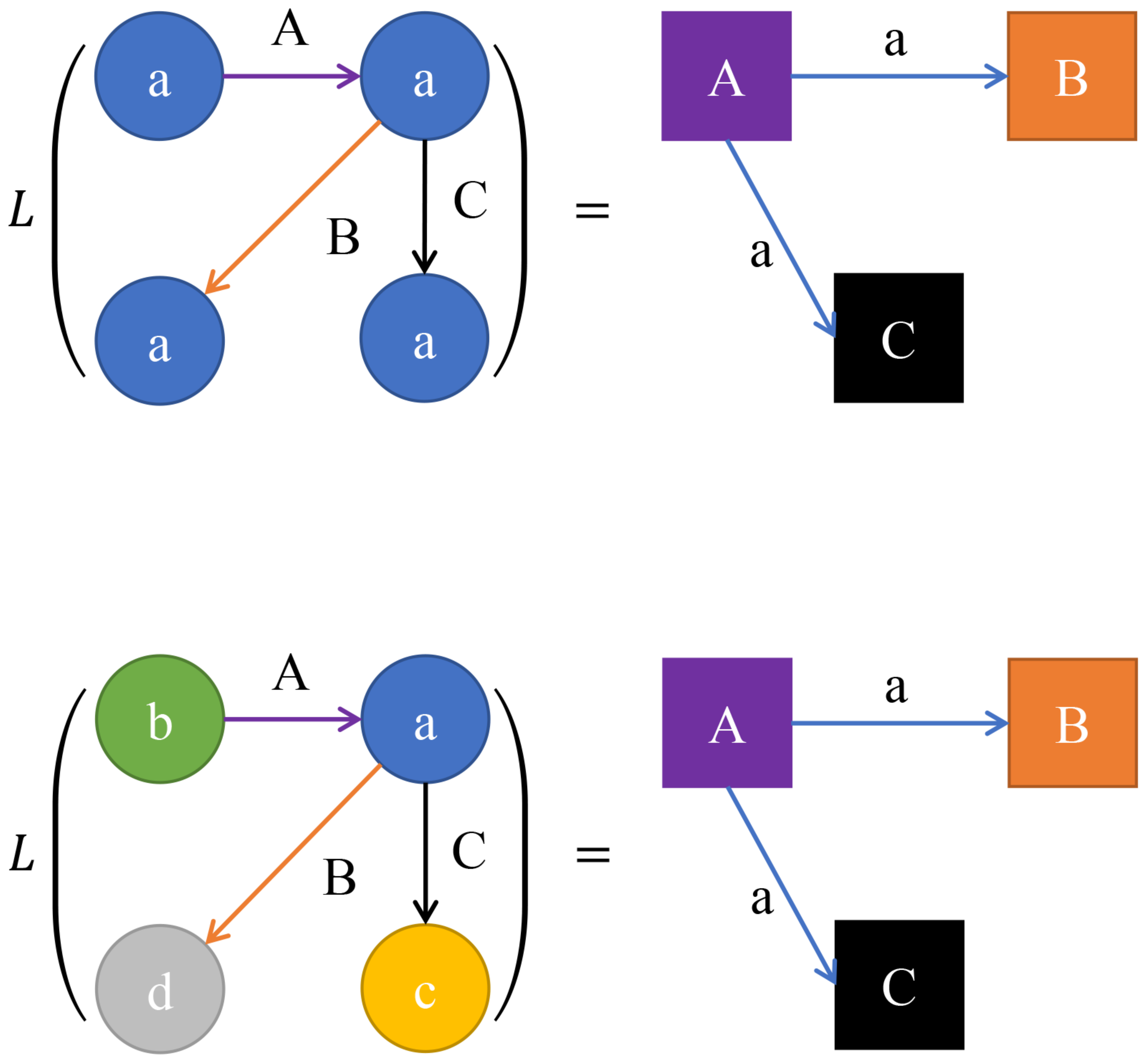}
        \caption{one vertex with 0 indegree, two vertices with 0 outdegree}
        \label{fig:claw_line2}
    \end{subfigure}
    \ \ 
    \begin{subfigure}{.23\textwidth}
        \centering
        \includegraphics[width=4.0cm]{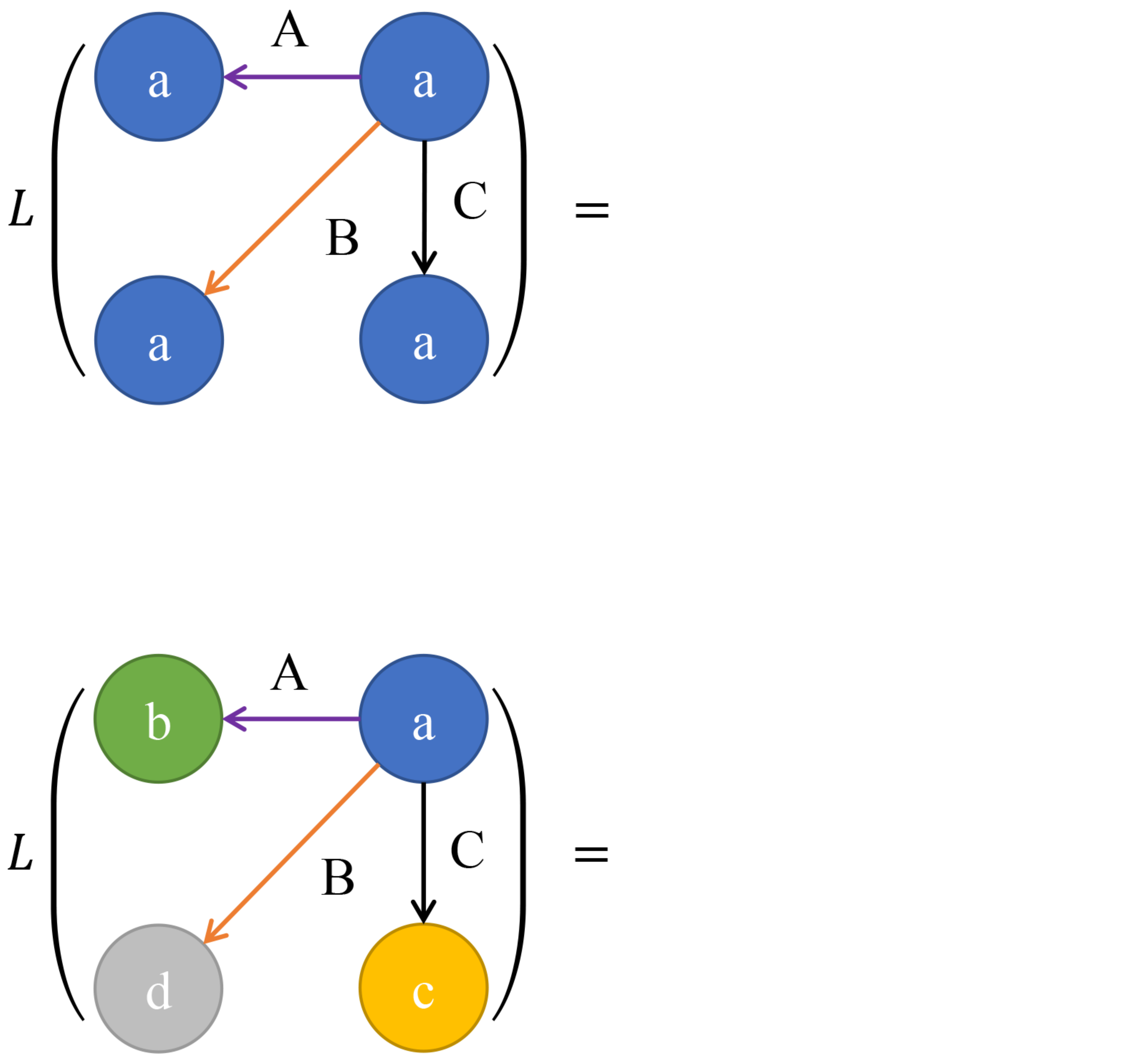}
        \caption{one vertex with 0 indegree, three vertices with 0 outdegree}
        \label{fig:claw_line5}
    \end{subfigure}
    \ \ 
    \begin{subfigure}{.23\textwidth}
        \centering
        \includegraphics[width=4.0cm]{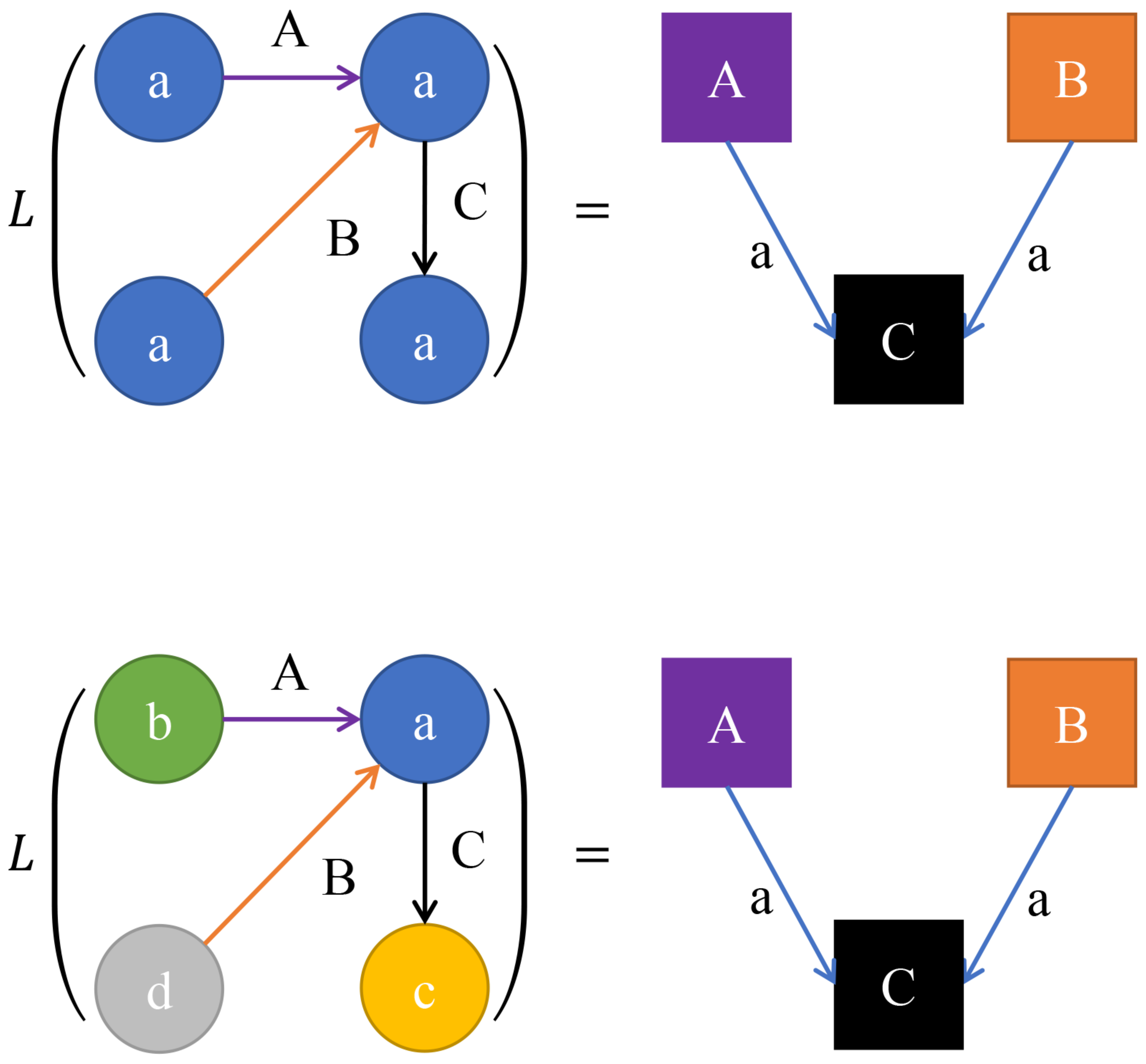}
        \caption{two vertices with 0 indegree, one vertex with 0 outdegree}
        \label{fig:claw_line3}
    \end{subfigure}
    \ \ 
    \begin{subfigure}{.24\textwidth}
        \centering
        \includegraphics[width=4.0cm]{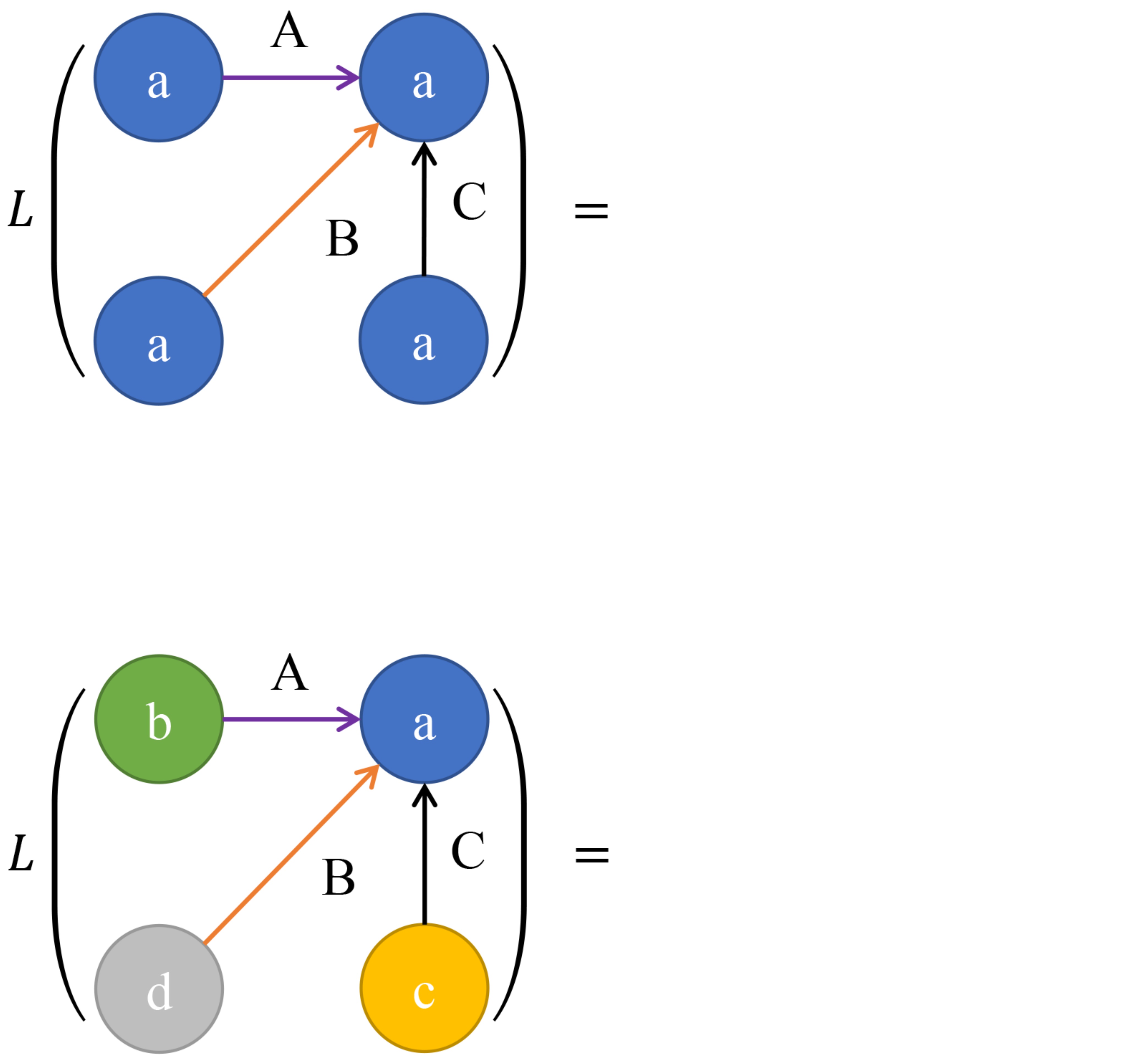}
        \caption{three vertices with 0 indegree, one vertex with 0 outdegree}
        \label{fig:claw_line4}
    \end{subfigure}
    \caption{Examples of directed 3-claws and line graphs, where lowercased letters correspond to vertex labels of the original graphs (and edge labels of the line graphs), capitalized letters correspond to edge labels of the original graphs (and vertex labels of the line graphs). Note that line graphs in (b) and (d) are empty.}
    \label{fig:claw_line}
\end{figure*}

In graph theory, the edge-to-vertex transform converts a graph to its line graph where original vertex properties are stored in edges, and original edge properties are stored in vertices in the transformed graph~\cite{harary1969graph}.
\begin{definition}[Edge-to-vertex transform]
\label{def:line}
A \textit{line graph} (also known as \textit{edge-to-vertex graph}) $\mathcal{H}$ of a graph $\mathcal{G}$ is obtained by a bijection $g: \mathcal{E}_{\mathcal{G}} \rightarrow \mathcal{V}_{\mathcal{H}}$, where
$g$ associates a vertex $v' \in \mathcal{V}_{\mathcal{H}}$ with each edge $e = g^{-1}(v') \in \mathcal{E}_{\mathcal{G}}$. And two vertices $u', v' \in \mathcal{V}_{\mathcal{H}}$ are connected as $(u', v')$ if and only if the destination of $d = g^{-1}(u')$ is the source of $e = g^{-1}(v')$.
Formally, we have: \\
\scalebox{0.88}{\parbox{1.13\linewidth}{
\begin{itemize}
    \item $\forall e=(u,v) \in \mathcal{E}_{\mathcal{G}}, 
    \mathcal{Y}_{\mathcal{G}}(e) = \mathcal{X}_{\mathcal{H}}(g(e))$,
    \item $\forall v' \in \mathcal{V}_{\mathcal{H}},
    \mathcal{X}_{\mathcal{H}}(v') = \mathcal{Y}_{\mathcal{G}}(g^{-1}(v'))$,
    \item $\forall d, e \in \mathcal{E}_{\mathcal{G}}, u' = g(d) \in \mathcal{V}_{\mathcal{H}}, v' = g(e) \in \mathcal{V}_{\mathcal{H}}, \\(d.\text{target}=e.\text{source}=v) \rightarrow (\mathcal{Y}_{\mathcal{H}}((u', v')) = \mathcal{X}_{\mathcal{G}}(v)))$,
    \item $\forall e'=(u', v') \in \mathcal{E}_{\mathcal{H}}, d=g^{-1}(u') \in \mathcal{V}_{\mathcal{G}}, e=g^{-1}(v') \in \mathcal{V}_{\mathcal{G}},\\ (d.\text{target} = e.\text{source}) \wedge (\mathcal{Y}_{\mathcal{H}}(e') = \mathcal{X}_{\mathcal{H}}(d.\text{target}))$.
\end{itemize}
}}
\end{definition}
We write $\mathcal{H}$ as $L(\mathcal{G})$ where $L: \mathcal{G} \rightarrow \mathcal{H}$ refers to the \textit{edge-to-vertex transform}.
Based on the definition, the number of vertices of $\mathcal{H}$ is the same as the number of edges of $\mathcal{G}$, and the number of edges of $\mathcal{H}$ equals to $\sum_{v \in \mathcal{V}_{\mathcal{G}}} d_v^- \cdot d_v^+$.
But there are two worst cases: empty or complete line graphs. 

\subsection{Non-injective Edge-to-vertex Transforms}

As the vertices of the line graph $\mathcal{H}$ corresponds to the edges of the original graph $\mathcal{G}$, some properties of $\mathcal{G}$ that depend only on adjacency between edges may be preserved as equivalent properties in $\mathcal{H}$ that depend on adjacency between vertices.
Intuitively, we have one question of whether we can transform the line graph back to its original graph by a function $L^{-1}: \mathcal{H} \rightarrow \mathcal{G}$.
The answer is no, because some information may be lost.
The reason behind is that some vertices in the original graph are located in claw structures.
One classic case is a 3-claw structure and a triangle having the same line graph in undirected scenarios.
In fact, directed cases are more general: an extreme case is that all directed 2-path structures (i.e., two vertices are connected by one edge) with the same edge label have the same line graph.
Two claws with similar structures but different labels may also result in the same line graph.
Figure~\ref{fig:claw_line} demonstrates the 3-claw examples, and more complicated structures can be enumerated by extending these claws.
All these instances have one thing in common: 
the information from those vertices with 0 indegree or 0 outdegree gets lost during the edge-to-vertex transform.
It is easy to get the Lemma~\ref{lemma:information_preserving} from Definition~\ref{def:line}.

\begin{lemma}
\label{lemma:information_preserving}
During the edge-to-vertex transform over a directed graph $\mathcal{G}$, the information of a vertex $v$ is preserved if and only if its indegree $d_v^-$ is nonzero and its outdegree $d_v^+$ is nonzero. In particular, there are $d_v^- \cdot d_v^+$ copies in the line graph $\mathcal{H} = L(\mathcal{G})$.
\end{lemma}

\subsection{Injective and Inversive Edge-to-vertex Transforms}
\label{sec:lossless_transform}

\citeauthor{liu2021graph}~(\citeyear{liu2021graph}) addressed the lost of information by introducing reversed edges with specific edge labels.
With the help of reversed edges, all (non-isolated) vertices have nonzero indegrees and nonzero outdegrees.
However, this strategy makes the graph and corresponding line graph extremely dense.
Assume a graph $\mathcal{G}$ and its line graph $\mathcal{H} = L(\mathcal{G})$, then the modified graph doubles the number of edges, and the corresponding line graph has $2|\mathcal{V}_{\mathcal{H}}|$ vertices and $\sum_{v \in \mathcal{V}_{\mathcal{G}}} (d_v^- + d_v^+)^2 \geq 4 |\mathcal{E}_{\mathcal{H}}|$ edges.
As a result, the line graph cannot be used directly in practice.
To say the least, it doubles the computation of graph convolutions and quadruples that of line graph convolutions.
Therefore, we propose our efficient solution to eliminate the 0-indegree vertices and 0-outdegree vertices by adding dummy edges starting from and sinking to one particular dummy node.


\begin{corollary}
\label{corollary:information_preserving}
Given a directed graph $\mathcal{G}$ with $n$ vertices, the modified graph $\mathcal{G}_{\varphi}$ involves one dummy node $\varphi$ and $2n$ dummy edges where 
this special dummy node connects every $v \in \mathcal{V}_{\mathcal{G}}$ by two dummy edges $(\varphi, v)$ and $(v, \varphi)$.
During the edge-to-vertex transform $L$ over $\mathcal{G}_{\varphi}$,
the information of each vertex $v \in \mathcal{G}$ is preserved. In particular,
there are $(d_v^- + 1) \cdot (d_v^+ + 1)$ copies in the line graph $\mathcal{H}_{\varphi} = L(\mathcal{G}_{\varphi})$.
\end{corollary}

According to Definition~\ref{def:line}, we have the statistics of the line graph $\mathcal{H}_{\varphi}$: \\
\scalebox{0.88}{\parbox{1.13\linewidth}{
\begin{align}
    |\mathcal{V}_{\mathcal{H}_{\varphi}}| &= m + 2n = |\mathcal{V}_{\mathcal{H}}| + 2n,  \label{eq:v_dummy} \\
    |\mathcal{E}_{\mathcal{H}_{\varphi}}| &= \sum_{v \in \mathcal{V}_{\mathcal{G}'}} d_v^- \cdot d_v^+ \nonumber \\
    &= n^2 + \sum_{v \in \mathcal{V}_{\mathcal{G}}} (d_v^- + 1) \cdot (d_v^+ + 1) \nonumber \\
    &= \textcolor[rgb]{0.997,0.439,0.000}{\sum_{v \in \mathcal{V}_{\mathcal{G}}} d_v^- \cdot d_v^+} + \textcolor[rgb]{0.412,0.412,0.412}{n^2} + \textcolor[rgb]{0.176,0.616,0.997}{n} + \textcolor[rgb]{0.900,0.900,0.003}{m} + \textcolor[rgb]{0.408,0.827,0.129}{m} \nonumber \\
    &= \textcolor[rgb]{0.997,0.439,0.000}{|\mathcal{E}_{\mathcal{H}}|} + \textcolor[rgb]{0.412,0.412,0.412}{n^2} + \textcolor[rgb]{0.176,0.616,0.997}{n} + \textcolor[rgb]{0.900,0.900,0.003}{m} + \textcolor[rgb]{0.408,0.827,0.129}{m}. \label{eq:e_dummy}
\end{align}
}}

\begin{figure}[!t]
    \centering
    \begin{subfigure}{0.48\textwidth}
        \centering
        \includegraphics[width=1.0\textwidth]{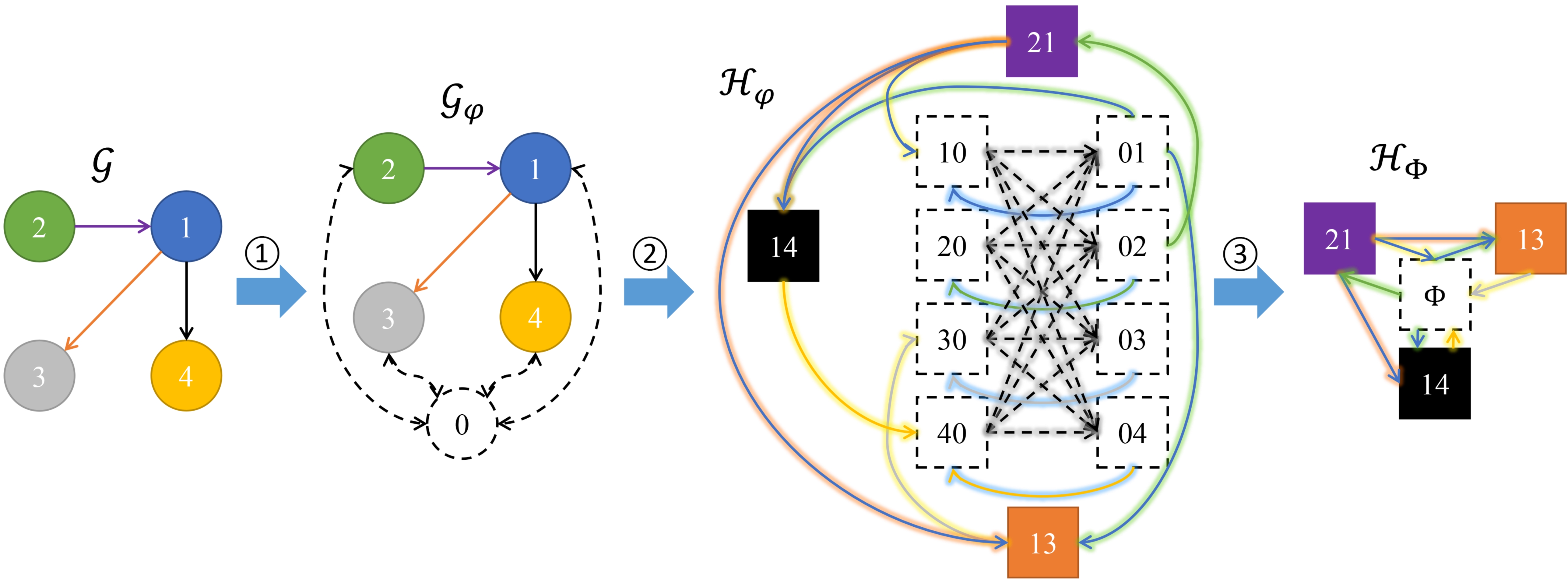}
        \caption{edge-to-vertex transform $L_{\Phi}$}
        \label{fig:evt}
    \end{subfigure}

    \begin{subfigure}{0.48\textwidth}
        \centering
        \includegraphics[width=1.0\textwidth]{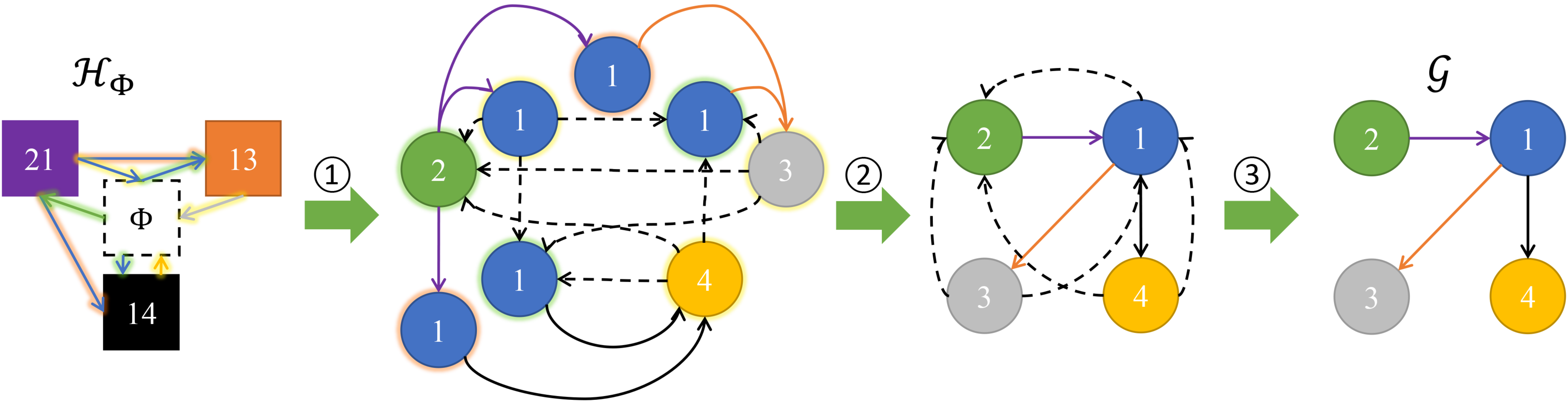}
        \caption{inverse edge-to-vertex transform $L_{\Phi}^{-1}$}
        \label{fig:ievt}
    \end{subfigure}
    \caption{Our proposed edge-to-vertex transforms $L_{\Phi}$ over 3-claws and the inverse transform $L_{\Phi}^{-1}$, where numbers located in circles indicates the original vertex ids, numbers located in squares indicates the pairs of original source id and original target id.}
    \label{fig:transform}
\end{figure}

As Corollary~\ref{corollary:information_preserving}, all original vertices are stored in the line graph.
But the costs are not negligible.
The number of vertices in the line graph obviously increases by $2n$.
But the worst thing is
dramatical growth of the number of edges in an additional complexity $\mathcal{O}(n^2)$.
To investigate further, we transform a directed 3-claw for demonstration.
As shown in Figure~\ref{fig:evt}, the step \textcircled{\raisebox{-.9pt} {2}} is the conventional edge-to-vertex transform over $\mathcal{G}_{\varphi}$.
The edges in its line graph $\mathcal{H}_{\varphi}$ can be divided into five categories:
\begin{itemize}[noitemsep]
    \item \textcolor[rgb]{0.997,0.439,0.000}{$|\mathcal{E}_{\mathcal{H}}|$} edges are connections through the original $n$ vertices, which can be preserved by the transform without the help of dummy node;
    \item \textcolor[rgb]{0.412,0.412,0.412}{$n^2$} edges are connections between dummy edges through the dummy node;
    \item \textcolor[rgb]{0.176,0.616,0.997}{$n$} edges are connections between dummy edges through the original $n$ vertices;
    \item \textcolor[rgb]{0.900,0.900,0.003}{$m$} edges start from the original edges and sink to the dummy edges;
    \item \textcolor[rgb]{0.408,0.827,0.129}{$m$} edges start from the dummy and sink to the original.
\end{itemize}
In fact, the \textcolor[rgb]{0.412,0.412,0.412}{$n^2$} edges plotted as \textcolor[rgb]{0.412,0.412,0.412}{dashed grey lines} are useless in original structure preserving since these edges does not store any graph-specific properties.
On the contrary, each of the \textcolor[rgb]{0.176,0.616,0.997}{$n$} edges marked in \textcolor[rgb]{0.176,0.616,0.997}{blue} maintains original vertex properties.
However, Corollary~\ref{corollary:information_preserving} confirms that there are $(d_v^- + 1) \cdot (d_v^+ + 1) = (d_v^- \cdot d_v^+ + d_v^- + d_v^+) + 1$ copies for an original vertex $v \in \mathcal{V}_{\mathcal{G}}$.
Since $d_v^- \cdot d_v^+ + d_v^- + d_v^+ > 0$ holds for connected components except an isolated point, it is also safe to remove these \textcolor[rgb]{0.176,0.616,0.997}{$n$} edges when $|\mathcal{E}_{\mathcal{G}}| = m > 0$.

After deleting the \textcolor[rgb]{0.412,0.412,0.412}{$n^2$} $+$ \textcolor[rgb]{0.176,0.616,0.997}{$n$} edges, we find there is no connections between dummy edges anymore.
Thus, we merge the corresponding $2n$ vertices as one dummy $\Phi$ and finally get the new transformed graph $\mathcal{H}_{\Phi}$ with \\
\scalebox{0.88}{\parbox{1.13\linewidth}{
\begin{align}
    |\mathcal{V}_{\mathcal{H}_{\Phi}}| &= m + 1 = |\mathcal{V}_{\mathcal{H}}| + 1, \label{eq:h_v} \\
    |\mathcal{E}_{\mathcal{H}_{\Phi}}| &= \sum_{v \in \mathcal{V}_{\mathcal{G}}} (d_v^- \cdot d_v^+ + d_v^- + d_v^+) = |\mathcal{E}_{\mathcal{H}}| + 2m, \label{eq:h_e} 
\end{align}
}}
where $\mathcal{H}$ is the line graph of the original graph $\mathcal{G}$.
The new edge-to-vertex transform $L_{\Phi}$ is shown in Figure~\ref{fig:evt}.
And we provide Algorithm~\ref{alg:evt} in Appendix~\ref{appendix:alg} for details.

Since no vertex information or edge information gets lost, the next is to find the inverse transform $L_{\Phi}^{-1}$.
Because both $G_{\varphi}$ and $H_{\Phi}$ contain a special dummy node, respectively, we consider the same strategy to merge vertices.
Before that, original vertex ids are assigned as edge ids for $\mathcal{H}_{\Phi}$. We are surprised to find that it is easy to get vertices and edges of $\mathcal{G}$ from $L(\mathcal{H}_{\Phi})$ after removing dummy edges.
The process is shown in Figure~\ref{fig:ievt},
and the algorithm is described in Algorithm~\ref{alg:ievt} in Appendix~\ref{appendix:alg}.
Theorem~\ref{theorem:ievt} shows that the \textit{inverse} of $L_{\Phi}$ always exists.

\begin{theorem}
\label{theorem:ievt}
For any $\mathcal{H}_{\Phi}$ transformed by $L_{\Phi}$ such that $\mathcal{H}_{\Phi} = L_{\Phi}(\mathcal{G})$, $L_{\Phi}^{-1}$ can always transform $\mathcal{H}_{\Phi}$ back to $\mathcal{G}$, i.e., $L_{\Phi}^{-1}(L_{\Phi}(\mathcal{G})) = \mathcal{G}$.
\end{theorem}

\begin{figure*}[!t]
    \centering
    \begin{subfigure}{.23\textwidth}
        \centering
        \includegraphics[width=4.0cm]{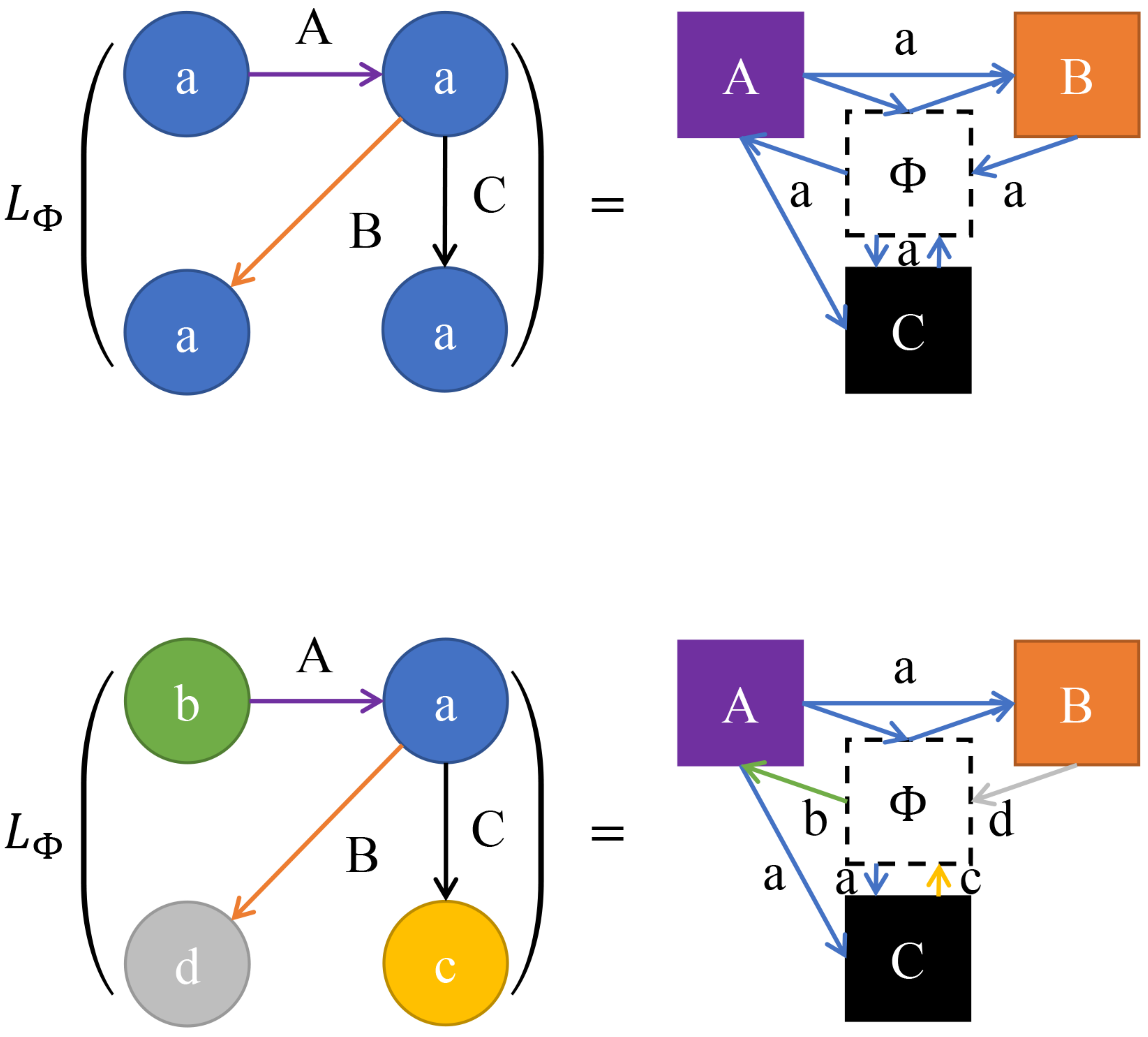}
        \caption{one vertex with 0 indegree, two vertices with 0 outdegree}
        \label{fig:claw_line7}
    \end{subfigure}
    \ 
    \begin{subfigure}{.23\textwidth}
        \centering
        \includegraphics[width=4.0cm]{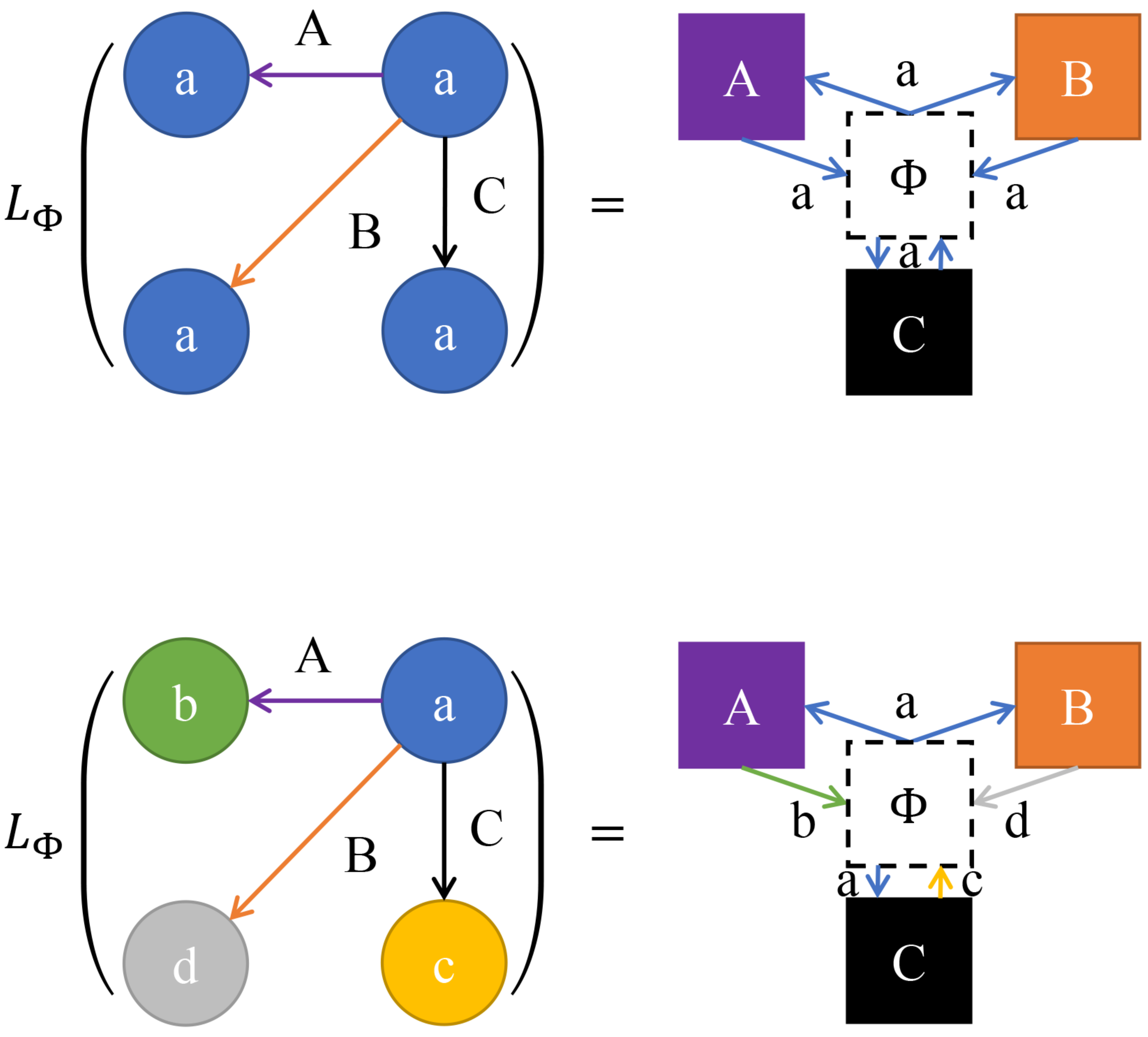}
        \caption{one vertex with 0 indegree, three vertices with 0 outdegree}
        \label{fig:claw_line10}
    \end{subfigure}
    \ 
    \begin{subfigure}{.23\textwidth}
        \centering
        \includegraphics[width=4.0cm]{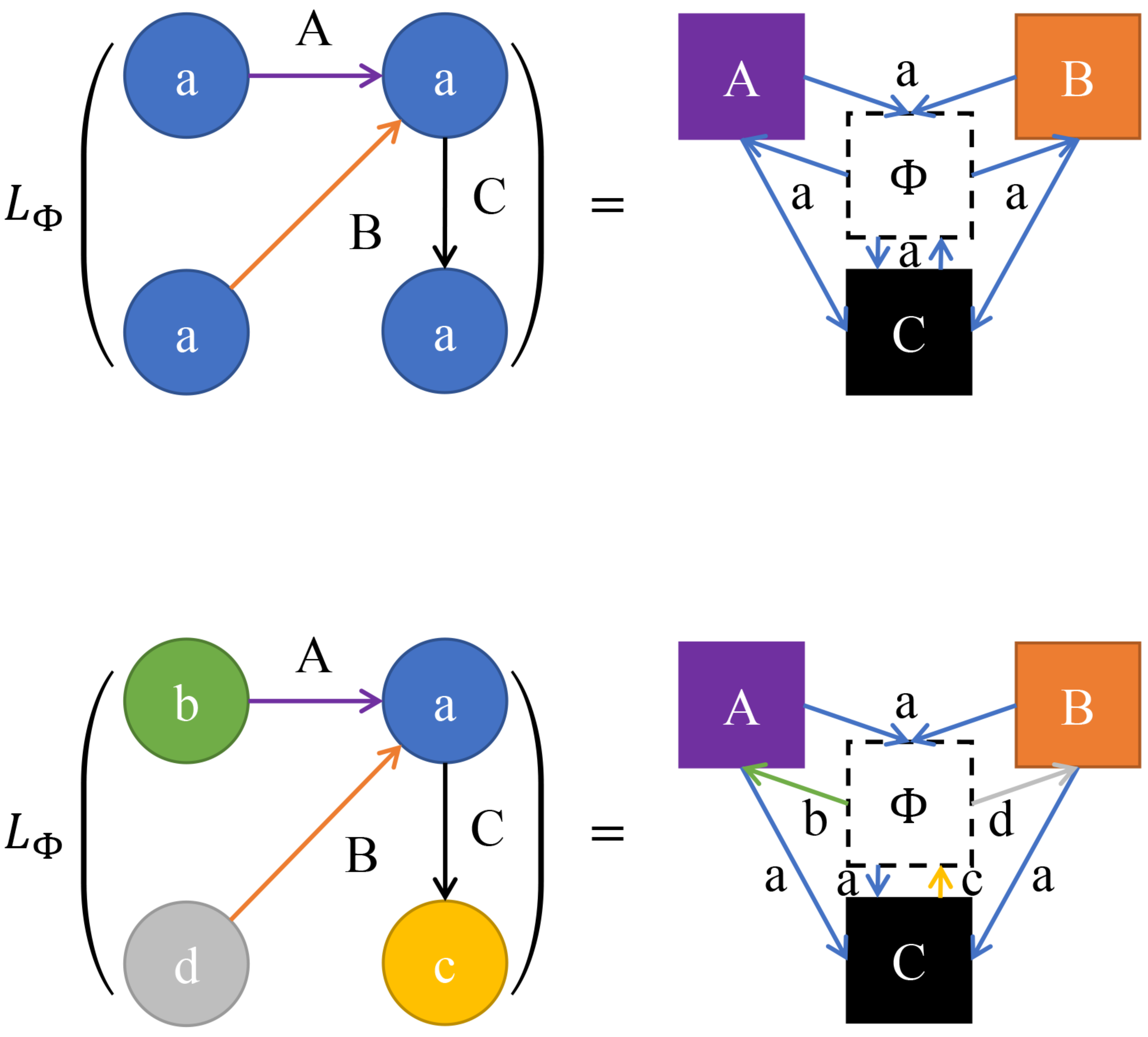}
        \caption{two vertices with 0 indegree, one vertex with 0 outdegree}
        \label{fig:claw_line8}
    \end{subfigure}
    \ 
    \begin{subfigure}{.24\textwidth}
        \centering
        \includegraphics[width=4.0cm]{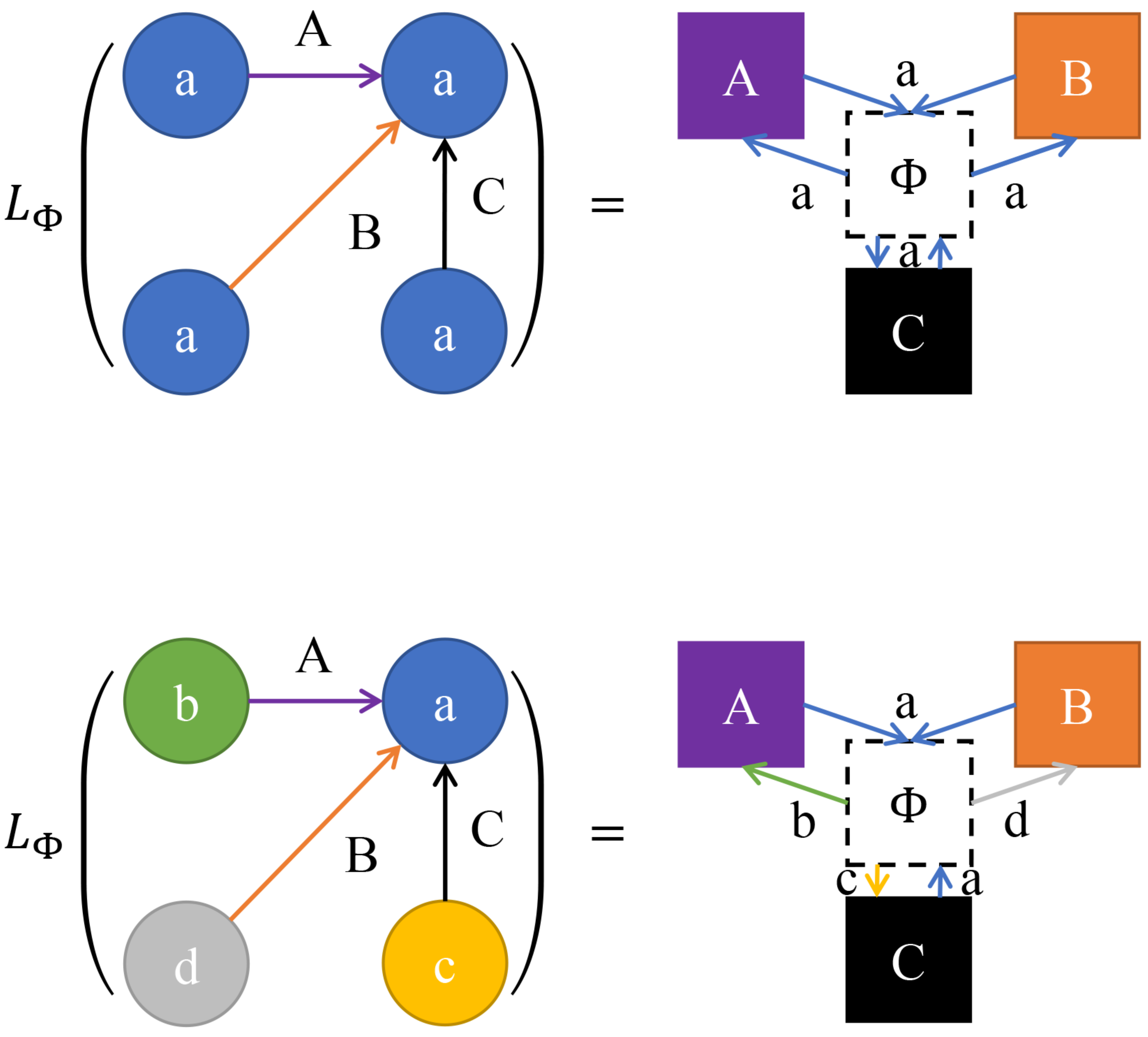}
        \caption{three vertices with 0 indegree, one vertex with 0 outdegree}
        \label{fig:claw_line9}
    \end{subfigure}
    \caption{Examples of directed 3-claws and new transformed edge-to-vertex graphs by $L_{\Phi}$.}
    \label{fig:claw_line_dummy}
\end{figure*}

\begin{proof}
\label{proof:ievt}
Let $\mathcal{G}'$ be $L_{\Phi}^{-1}(L_{\Phi}(\mathcal{G}))$,
$\mathcal{G}'_{\scriptsize{\textcircled{\raisebox{-.9pt} {i}}}}$ be the resulting graph of the step $\textcircled{\raisebox{-.9pt} {i}}$ ($\text{i} \in \{1,2,3\}$) of $L_{\Phi}^{-1}$.

Each vertex $w$ of $\mathcal{H}_{\Phi}$ except the dummy $\Phi$ corresponds to one edge (e.g., $e=(u, v)$) of $\mathcal{G}$, and $\Phi$ connects $w$ with two edges where $(\Phi, w)$ is associated with label $\mathcal{X}_{\mathcal{G}}(u)$ and $(w, \Phi)$ is associated with label $\mathcal{X}_{\mathcal{G}}(v)$.
After the step $\textcircled{\raisebox{-.9pt} {1}}$, $(\Phi, w)$ is transformed as one vertex with id $u$ and label $\mathcal{X}_{\mathcal{G}}(u)$, and $(w, \Phi)$ is converted as one vertex with id $v$ and label $\mathcal{X}_{\mathcal{G}}(v)$.
Similarly, other edges not connecting $\Phi$ can also be transformed as vertices.
We have:
\scalebox{0.88}{\parbox{1.13\linewidth}{
\begin{align}
    |\mathcal{V}_{\mathcal{G}'_{\tiny{\textcircled{\raisebox{-.9pt} {1}}}}}| = \sum_{v \in \mathcal{V}_{\mathcal{G}}} (d_v^- \cdot d_v^+ + d_v^- + d_v^+), \nonumber
\end{align}
}}
where vertex $v$ in $\mathcal{G}$ appears $d_v^- \cdot d_v^+ + d_v^- + d_v^+$ times (that is nonzero) in $\mathcal{G}'_{\tiny{\textcircled{\raisebox{-.9pt} {1}}}}$.
On the other hand, the indegree and the outdegree of $w$ are $d_u^- + 1$ and $d_v^+ + 1$, respectively.
Thus, we get:
\scalebox{0.88}{\parbox{1.13\linewidth}{
\begin{align}
    |\mathcal{E}_{\mathcal{G}'_{\tiny{\textcircled{\raisebox{-.9pt} {1}}}}}| &= \sum_{w \in \mathcal{H}_{\Phi}} d_w^- \cdot d_w^+ \nonumber \\
    &= m^2 + \sum_{(u, v) \in \mathcal{E}_{\mathcal{G}}} (d_u^- + 1) \cdot (d_v^+ + 1) \nonumber \\
    &= 2 \sum_{v \in \mathcal{V}_{\mathcal{G}}}d_v^- \cdot d_v^+ + \sum_{(u, v) \in \mathcal{E}_{\mathcal{G}}} d_u^- \cdot d_v^+ + m^2 + m, \nonumber
\end{align}
}}
where $2 \sum_{v \in \mathcal{V}_{\mathcal{G}}}d_v^- \cdot d_v^+$ says inedge $(?, v)$ is copied $d_v^+$ times thanks to $(w, \Phi)$ and outedge $(v, ?)$ is copied $d_v^+$ times due to $(\Phi, w)$, $\sum_{(u, v) \in \mathcal{E}_{\mathcal{G}}} d_u^- \cdot d_v^+$ refers to the transformed edges by $L(L(\mathcal{G}))$, $m^2$ corresponds to dummy edges, $m$ indicates original edges recovered by $(\Phi, w)$ and $(w, \Phi)$ directly.

As we store vertex ids and vertex labels from $\mathcal{G}$ in edges of $\mathcal{H}_{\Phi}$, we get $\mathcal{V}_{\mathcal{G}'_{\scriptsize{\textcircled{\raisebox{-.9pt} {2}}}}}$ is exactly the same as $\mathcal{V}_{\mathcal{G}}$ after merging vertices with same information.
Since the step $\textcircled{\raisebox{-.9pt} {3}}$ does not remove vertices further,  for $\mathcal{V}_{\mathcal{G}'} = \mathcal{V}_{\mathcal{G}}$ holds.

Edges are merged based on source ids and target ids in the step $\textcircled{\raisebox{-.9pt} {2}}$, and additional dummy edges are removed in the step $\textcircled{\raisebox{-.9pt} {3}}$.
Therefore, $\mathcal{E}_{\mathcal{G}} \subseteq \mathcal{E}_{\mathcal{G}'} \subseteq \mathcal{E}_{\mathcal{G}'_{\tiny{\textcircled{\raisebox{-.9pt} {1}}}}}$.
Next, we prove $\mathcal{E}_{\mathcal{G}} = \mathcal{E}_{\mathcal{G}'}$ by $\forall (u'', v'') \in \mathcal{E}_{L(L(\mathcal{G}))}, \mathcal{Y}_{L(L(\mathcal{G}))}((u'', v'')) \subseteq \mathcal{Y}_{\mathcal{G}}((u''.id, v''.id))$.
\citeauthor{beineke1968derived}~(\citeyear{beineke1968derived}) and \citeauthor{beineke1982connection}~(\citeyear{beineke1982connection}) discussed $L(L(\mathcal{G}))$ in more detail.
We simply prove in another direction.
Let $\mathcal{S}_{\mathcal{G}} = \{u' \in \mathcal{V}_{\mathcal{G}} | d_{u'}^- = 0\}$,  $\mathcal{T}_{\mathcal{G}} = \{v' \in \mathcal{V}_{\mathcal{G}} | d_{v'}^+ = 0\}$, and $\mathcal{U}_{\mathcal{G}} = \{v | d_v^- \cdot d_v^+ > 0\} = \mathcal{V}_{\mathcal{G}} - \mathcal{S}_{\mathcal{G}} \cup \mathcal{T}_{\mathcal{G}}$
then any edge $e$ in $\mathcal{S}_{\mathcal{G}} \times \mathcal{T}_{\mathcal{G}}$ corresponds to a isolated vertex in $L(\mathcal{G})$ so that it gets lost in $L(L(\mathcal{G}))$, i.e., $\phi \subseteq \mathcal{Y}_{\mathcal{G}}(e)$.
For other vertices $\mathcal{U}_{\mathcal{G}} = \{v | d_v^- + d_v^+ > 0\}$, any edge $e=(u, v) \in \mathcal{U}_{\mathcal{G}} \times \mathcal{U}_{\mathcal{G}} \cap \mathcal{E}_{\mathcal{G}}$ is located in at least one \textit{3-diwalk}, e.g., $s \rightarrow u \rightarrow v \rightarrow t$ where $s, t \in \mathcal{V}_{\mathcal{G}}$.
Then $L(\mathcal{G})$ converts this 3-diwalk to a \textit{2-diwalk} $su' \rightarrow uv' \rightarrow vt'$, and $L(L(\mathcal{G}))$ finally results in a \textit{1-diwalk} $u'' \rightarrow v''$ where $u''$ corresponds to $u$ and $v''$ corresponds to $v$, i.e., $\mathcal{Y}_{L(L(\mathcal{G}))}((u'', v'')) = \mathcal{Y}_{\mathcal{G}}((u, v))$.
Hence, $\forall (u'', v'') \in \mathcal{E}_{L(L(\mathcal{G}))}, \mathcal{Y}_{L(L(\mathcal{G}))}((u'', v'')) \subseteq \mathcal{Y}_{\mathcal{G}}((u''.id, v''.id))$.
That is to say, both $2 \sum_{v \in \mathcal{V}_{\mathcal{G}}}d_v^- \cdot d_v^+$ edges and $\sum_{(u, v) \in \mathcal{E}_{\mathcal{G}}} d_u^- \cdot d_v^+$ edges are going to be merged into $m$ edges in the step $\textcircled{\raisebox{-.9pt} {2}}$ and $|\mathcal{E}_{\mathcal{G}'}| = m$. Considering $\mathcal{E}_{\mathcal{G}} \subseteq \mathcal{E}_{\mathcal{G}'}$ and $m = \mathcal{E}_{\mathcal{G}}$, we have $\mathcal{E}_{\mathcal{G}'} = \mathcal{E}_{\mathcal{G}}$.
Clearly, $\mathcal{G}' = \mathcal{G}$. And this finishes the proof.
\end{proof}

Considering the symmetry between $\mathcal{G}_{\varphi}$ and $\mathcal{H}_{\Phi}$ where the dummy node serves as a hub to connect all other vertices, 
we call $\mathcal{G}_{\varphi}$ and $\mathcal{H}_{\Phi}$ \textit{conjugate} of each other.

Figure~\ref{fig:claw_line_dummy} illustrates the new transformed edge-to-vertex graphs over 3-claws.
As observed, all vertex information is preserved in new transformed graphs, the union of line graphs and six heterogeneous edges connecting dummy $\Phi$.

\subsection{Transforms and Morphisms}

As $L_{\Phi}$ is a structure-preserving map, we are also interested in the connections between $L_{\Phi}$ and morphisms.
In graph theory, one of the most important bijective morphisms is the isomorphism.

\begin{definition}[Isomorphism]
\label{def:iso}
A graph $\mathcal{G}_1$ is \textit{isomorphic} to a graph $\mathcal{G}_2$ if there is a bijection $f: \mathcal{V}_{\mathcal{G}_1} \rightarrow \mathcal{V}_{\mathcal{G}_2}$ such that: \\
\scalebox{0.88}{\parbox{1.13\linewidth}{
\begin{itemize}
    \item $\forall v \in \mathcal{V}_{\mathcal{G}_1}, \mathcal{X}_{\mathcal{G}_1}(v) = \mathcal{X}_{\mathcal{G}_2}(f(v))$, 
    \item $\forall v' \in \mathcal{V}_{\mathcal{G}_2}, \mathcal{X}_{\mathcal{G}_2}(v') = \mathcal{X}_{\mathcal{G}_1}(f^{-1}(v'))$,
    \item $\forall (u, v) \in \mathcal{E}_{\mathcal{G}_1}, \mathcal{Y}_{\mathcal{G}_1}((u, v)) = \mathcal{Y}_{\mathcal{G}_2}((f(u), f(v)))$,
    \item $\forall (u', v') \in \mathcal{E}_{\mathcal{G}_2}, \mathcal{Y}_{\mathcal{G}_2}((u', v')) = \mathcal{Y}_{\mathcal{G}_1}((f^{-1}(u'), f^{-1}(v')))$.
\end{itemize}
}}
\end{definition}
We write $\mathcal{G}_1 \simeq \mathcal{G}_2$ for such the isomorphic property, and $F$ is named as an \textit{isomorphism}. For two isomorphic graphs, they are also regarded as both $\textit{permutation}$.



$L_{\Phi}$ and $L_{\Phi}^{-1}$ are also morphisms from graphs to graphs.
In particular, $L_{\Phi}^{-1}$ is an \textit{epimorphism} based on Theorem~\ref{theorem:ievt}.
The following Propositions~\ref{proposition:monomorphism} shows that $L_{\Phi}$ is a \textit{monomorphism}.

\begin{proposition}
\label{proposition:monomorphism}
$L_{\Phi}$ is a monomorphism.
\end{proposition}
\begin{proof}
\label{proof:monomorphism}
To prove that $L_{\Phi}$ is a monomorphism, we seek to show $L_{\Phi}(\mathcal{G}_1) = L_{\Phi}(\mathcal{G}_2) \rightarrow \mathcal{G}_1 = \mathcal{G}_2$ for any $\mathcal{G}_1, \mathcal{G}_2$.
Theorem~\ref{theorem:ievt} shows that $L_{\Phi}^{-1}$ always exists, and $L_{\Phi}^{-1}(L_{\Phi}(\mathcal{G}_1)) = \mathcal{G}_1$, $L_{\Phi}^{-1}(L_{\Phi}(\mathcal{G}_2)) = \mathcal{G}_2$ always holds.
Given $L_{\Phi}(\mathcal{G}_1) = L_{\Phi}(\mathcal{G}_2)$, we have $L_{\Phi}^{-1}(L_{\Phi}(\mathcal{G}_1)) = L_{\Phi}^{-1}(L_{\Phi}(\mathcal{G}_2))$, which implies $\mathcal{G}_1 = \mathcal{G}_2$. And this finishes the proof.
\end{proof}

Propositions~\ref{proposition:monomorphism} is important because we can apply the edge-to-vertex transform $L_{\Phi}$ before other morphisms and functions without breaking properties, such as Corollary~\ref{corollary:isomorphism_evt}.

\begin{corollary}
\label{corollary:isomorphism_evt}
Isomorphisms hold after $L_{\Phi}$.
\end{corollary}

More generally, permutation-invariant functions are expected to acquire the same outputs among permutations.
In view of Corollary~\ref{corollary:isomorphism_evt}, we get the following consequence:

\begin{corollary}
\label{corollary:permutation_evt}
If a function $h$ is permutation-invariant, then $h \circ L_{\Phi}$ is also permutation-invariant.
\end{corollary}

This helps us to apply vertex-centric graph kernel functions and graph neural networks to learn edge-centric representations.
And all above indicate that the graph with a dummy node $\mathcal{G}_{\varphi}$ is better to learn structure information than $\mathcal{G}$.
\section{Methodology}

In this section, we extend effective machine learning kernel functions and deep graph neural networks with dummy nodes and our proposed edge-to-vertex transform $L_{\Phi}$.

\subsection{Extensions of Graph Kernel Functions}
\label{section:extension_gk}
A graph kernel (GK) is a symmetric, positive semi-definite function defined on the graph space.
It is usually expressed as an inner product in Hilbert space~\cite{kriege2020asurvey} such that $k(\mathcal{G}_1, \mathcal{G}_2) = \langle h(\mathcal{G}_1), h(\mathcal{G}_2) \rangle$, where $k$ is the kernel function and $h$ is the permutation-invariant function from graph space to Hilbert space.
In general, $k$ measures the similarity between two graphs, and the graph similarity is directly related to graph comparison in machine learning.
In this paper, we aim to explore the power of dummy nodes, so we adopt graph-structure-sensitive and attribute-sensitive kernels, 
like Weisfeiler-Lehman Subtree Kernel (WL)~\cite{shervashidze2011weisfeiler}.
We generalize these kernels with dummy nodes (denoted as $k_{\varphi}$) and the edge-to-vertex transform (denoted as $k_{\Phi}$): \\
\scalebox{0.88}{\parbox{1.13\linewidth}{
\begin{align}
    k_{\varphi}(\mathcal{G}_1, \mathcal{G}_2) &= k(\mathcal{G}_{1}, \mathcal{G}_{2}) + k(\mathcal{G}_{{\varphi}_1}, \mathcal{G}_{{\varphi}_2}) \nonumber \\
    &= \langle h(\mathcal{G}_{1}), h(\mathcal{G}_{2}) \rangle + \langle h(\mathcal{G}_{{\varphi}_1}), h(\mathcal{G}_{{\varphi}_2}) \rangle, \label{eq:kernel_varphi} \\
    k_{\Phi}(\mathcal{G}_1, \mathcal{G}_2) &= k(\mathcal{G}_{1}, \mathcal{G}_{2}) + k(\mathcal{H}_{{\Phi}_1}, \mathcal{H}_{{\Phi}_2}) \nonumber \\
    &= \langle h(\mathcal{G}_{1}), h(\mathcal{G}_{2}) \rangle + \langle h(L_{\Phi}(\mathcal{G}_{1})), h(L_{\Phi}(\mathcal{G}_{2})) \rangle, \label{eq:kernel_Phi}
\end{align}
}}
where $\mathcal{G}_{{\varphi}_1}$ and $\mathcal{G}_{{\varphi}_2}$ respectively correspond to graphs $\mathcal{G}_{1}$ and $\mathcal{G}_{2}$ with a dummy node $\varphi$ and dummy edges, and $\mathcal{H}_{{\Phi}_1} = L_{\Phi}(\mathcal{G}_{1})$ and $\mathcal{H}_{{\Phi}_2} = L_{\Phi}(\mathcal{G}_{2})$ are transformed by the proposed $L_{\Phi}$, each of which contains a dummy node $\Phi$.
We add the $k(\mathcal{G}_{1}, \mathcal{G}_{2})$ term in $k_{\varphi}$ and $k_{\Phi}$ to enforce the kernel functions to pay more attention to original structures; otherwise, the dummy node and dummy edges may bring some side effects.

\subsection{Extensions of Graph Neural Networks}

Many graph neural networks (GNNs) have been proposed to learn graph structures, such as GCN~\cite{kipf2017semi}, GraphSAGE~\cite{hamilton2017inductive}, GIN~\cite{xu2019how}.
Most can be unified in the \textit{Message Passing} framework~\cite{gilmer2017neural}: \\
\scalebox{0.88}{\parbox{1.13\linewidth}{
\begin{align}
    \bm \Delta_v^{(t+1)} &= Aggregate(\{ Message(\bm x_{v}^{(t)}, \bm x_{u}^{(t)}, \bm y_{(u, v)})| u \in \mathcal{N}_{v}\}),  \nonumber \\
    \bm x_{v}^{(t+1)} &= Update(\bm x_{v}, \bm \Delta_v^{(t+1)}), \nonumber 
\end{align}
}}
where $\bm x_{v}^{(t)}$ is the hidden state of vertex $v$ at the $t$-th layer network, $\mathcal{N}_{v}$ is $v$'s neighbor collection, $\bm y_{(u, v)}$ is the edge tensor for $(u, v)$, $\bm \Delta_v^{(t+1)}$ is the aggregated message from neighbors, $Aggregate$ is a permutation-invariant functions (e.g., $Sum$) to aggregate all messages as one, and $Update$ is a combination function (e.g., $Add$) to fuse $\Delta_v^{(t+1)}$ and $\bm x_{v}^{(t)}$ as the new state $\bm x_{v}^{(t+1)}$.
After introducing a dummy node to $\mathcal{G}$, the neighbor collection is extended with a dummy node $\varphi$, and the dummy node serves to aggregate the global graph information in turn.
Similarly, we feed $L_{\Phi}(\mathcal{G})$ as the input to learn the graph representation, where such neural networks actually model the edges in $\mathcal{G}$.

\subsection{Efficiency of Extensions}

\subsubsection{Learning with $\mathcal{G}_{\varphi}$}
\revisexin{
Even though we introduce one particular dummy node and $2n$ dummy edges to $\mathcal{G}$ (where $n$ is the vertex size of $\mathcal{G}$), kernels’ complexities when the size respect graphlets or tuples is not too large, denoted as $k$ in Sec.~\ref{sec:introduction}. GNNs yield additional computation of the $2n$ dummy edges, but it is still efficient because we do not involve additional operations upon existing $n$ vertices and $2n$ is usually much less than the existing $m$ edges.}

\subsubsection{Learning with $\mathcal{H}_{\Phi}$}
\revisexin{The overhead of utilizing $\mathcal{H}_{\Phi}$ include the cost to obtain and the cost to use. Sec.~\ref{sec:lossless_transform} analyzes the former part: $\mathcal{H}$ is constructed in $\sum_{v \in \mathcal{V}_{\mathcal{G}}}{d_v^- \cdot d_v^+} = \mathcal{O}(\tilde{d} \cdot m)$, and building $\mathcal{H}_{\Phi}$ also requires $\mathcal{O}((\tilde{d}+2) \cdot m) = \mathcal{O}(\tilde{d} \cdot m)$, where $\tilde{d}$ is a coefficient depending on graph structures, which is usually small and bounded by the maximum outdegree of $\mathcal{G}$.
After transforming $\mathcal{G}$ to $\mathcal{H}_{\Phi}$, kernels' complexities increase from $\mathcal{O}(k \cdot n^{k+1})$ to $\mathcal{O}(k \cdot m^{k+1})$, and $k$-layer GNNs' complexities increase from $\mathcal{O}(k \cdot m)$ to $\mathcal{O}(k \cdot \tilde{d} \cdot m)$.
Kernel complexities dramatically increase, but using $\mathcal{H}_{\Phi}$ looks more acceptable compared with increasing $k$; GNNs benefit from parallelization, so the running time also increases linearly. Overall, using $\mathcal{H}_{\Phi}$ as input is still efficient.
}

\begin{table*}[!ht]
    \scriptsize
    \centering
    \setlength\tabcolsep{1pt}
    \begin{tabular}{ll|ccc|ccc|ccc|ccc}
    \toprule
        \multicolumn{2}{c|}{\multirow{2}{*}{Models}}
        & \multicolumn{3}{c|}{PROTEINS} & \multicolumn{3}{c|}{D\&D} & \multicolumn{3}{c|}{NCI109} & \multicolumn{3}{c}{NCI1} \\
        & & $\mathcal{G}$ & $\mathcal{G}_{\varphi}$ & $\mathcal{H}_{\Phi}$ & $\mathcal{G}$ & $\mathcal{G}_{\varphi}$ & $\mathcal{H}_{\Phi}$ & $\mathcal{G}$ & $\mathcal{G}_{\varphi}$ & $\mathcal{H}_{\Phi}$ & $\mathcal{G}$ & $\mathcal{G}_{\varphi}$ & $\mathcal{H}_{\Phi}$ \\
        \midrule
        \multirow{8}{*}{Kernel} &
        \multirow{1}{*}{SP} & 
        73.48$\pm$\tiny{3.93} & \bf 74.20$\pm$\tiny{3.23} & 73.39$\pm$\tiny{3.04} & 
        80.50$\pm$\tiny{3.66} & 79.58$\pm$\tiny{3.91} & \bf 81.51$\pm$\tiny{3.91} & 
        73.65$\pm$\tiny{2.34} & 73.84$\pm$\tiny{2.07} & 74.11$\pm$\tiny{2.22} & 
        74.18$\pm$\tiny{1.67} & 74.70$\pm$\tiny{1.74} & 74.40$\pm$\tiny{1.74} \\
        & \multirow{1}{*}{GR} & 
        70.45$\pm$\tiny{6.54} & 74.20$\pm$\tiny{4.44} & 73.66$\pm$\tiny{4.00} & 
        78.82$\pm$\tiny{3.83} & 79.66$\pm$\tiny{5.18} & 78.82$\pm$\tiny{3.87} & 
        66.45$\pm$\tiny{2.14} & 72.46$\pm$\tiny{2.51} & 71.81$\pm$\tiny{2.69} & 
        65.16$\pm$\tiny{2.30} & 73.04$\pm$\tiny{1.81} & 71.07$\pm$\tiny{1.47} \\
        & \multirow{1}{*}{WLOA} & 
        72.59$\pm$\tiny{2.46} & 73.84$\pm$\tiny{3.29} & 74.02$\pm$\tiny{3.47} & 
        79.24$\pm$\tiny{3.61} & 79.24$\pm$\tiny{3.81} & 78.57$\pm$\tiny{3.59} & 
        85.43$\pm$\tiny{1.51} & 84.61$\pm$\tiny{1.52} & 84.81$\pm$\tiny{1.11} & 
        85.96$\pm$\tiny{1.82} & 86.33$\pm$\tiny{1.77} & \bf 86.37$\pm$\tiny{1.75} \\
        & \multirow{1}{*}{1-WL} & 
        71.79$\pm$\tiny{4.52} & 73.30$\pm$\tiny{4.14} & 73.48$\pm$\tiny{5.02} & 
        80.50$\pm$\tiny{4.43} & 81.26$\pm$\tiny{4.08} & 80.42$\pm$\tiny{3.85} & 
        \bf 85.54$\pm$\tiny{1.34} & 83.74$\pm$\tiny{0.94} & 84.37$\pm$\tiny{1.02} & 
        85.13$\pm$\tiny{1.69} & 84.87$\pm$\tiny{1.77} & 85.38$\pm$\tiny{1.21} \\
        \cline{2-14} & \multirow{1}{*}{2-WL} & 
        74.11$\pm$\tiny{5.19} & 75.27$\pm$\tiny{4.67} & OOM & 
        OOM & OOM & OOM & 
        68.09$\pm$\tiny{1.55} & 68.38$\pm$\tiny{1.21} & 72.24$\pm$\tiny{1.85} & 
        67.71$\pm$\tiny{1.33} & 67.49$\pm$\tiny{1.45} & 69.00$\pm$\tiny{2.34} \\
        & \multirow{1}{*}{$\delta$-2-WL} & 
        74.20$\pm$\tiny{4.98} & 74.82$\pm$\tiny{4.16} & OOM & 
        OOM & OOM & OOM & 
        68.00$\pm$\tiny{1.94} & 68.26$\pm$\tiny{1.59} & 70.34$\pm$\tiny{1.87} & 
        67.32$\pm$\tiny{1.34} & 67.37$\pm$\tiny{1.40} & 69.20$\pm$\tiny{2.18} \\
        & \multirow{1}{*}{$\delta$-2-LWL} & 
        73.66$\pm$\tiny{5.10} & 74.37$\pm$\tiny{3.34} & 74.11$\pm$\tiny{3.72} & 
        77.06$\pm$\tiny{5.99} & 77.31$\pm$\tiny{5.98} & 79.41$\pm$\tiny{5.28} & 
        84.20$\pm$\tiny{1.44} & 83.12$\pm$\tiny{1.34} & 83.82$\pm$\tiny{1.06} & 
        85.40$\pm$\tiny{1.28} & 84.06$\pm$\tiny{1.54} & 85.40$\pm$\tiny{1.51} \\
        & \multirow{1}{*}{$\delta$-2-LWL$^+$} & 
        78.12$\pm$\tiny{4.75} & 83.48$\pm$\tiny{4.34} & \bf 84.55$\pm$\tiny{3.62} & 
        77.14$\pm$\tiny{6.05} & 77.56$\pm$\tiny{6.30} & \bf 79.58$\pm$\tiny{6.24} & 
        88.79$\pm$\tiny{0.94} & \bf 89.42$\pm$\tiny{1.37} & 88.57$\pm$\tiny{0.97} & 
        91.92$\pm$\tiny{1.93} & \bf 93.67$\pm$\tiny{0.84} & 91.65$\pm$\tiny{1.96} \\
        \hline
        \hline
        \multirow{7}{*}{Network} &
        \multirow{1}{*}{GraphSAGE} & 
        73.48$\pm$\tiny{5.66} & 73.93$\pm$\tiny{5.68} & - & 
        77.73$\pm$\tiny{4.66} & 78.91$\pm$\tiny{4.59} & - &
        73.38$\pm$\tiny{2.68} & 74.13$\pm$\tiny{2.30} & - &
        73.82$\pm$\tiny{2.17} & 74.31$\pm$\tiny{2.27} & - \\
        & \multirow{1}{*}{GCN} & 
        72.95$\pm$\tiny{3.88} & 74.02$\pm$\tiny{3.82} & - & 
        72.77$\pm$\tiny{4.62} & \bf 80.76$\pm$\tiny{5.37} & - &
        50.34$\pm$\tiny{2.69} & 51.67$\pm$\tiny{5.52} & - &
        61.75$\pm$\tiny{11.1} & 68.95$\pm$\tiny{10.8} & - \\
        & \multirow{1}{*}{GIN} & 
        73.84$\pm$\tiny{4.46} & 74.11$\pm$\tiny{4.12} & - &
        76.97$\pm$\tiny{3.87} & 77.65$\pm$\tiny{3.46} & - & 
        72.61$\pm$\tiny{2.37} & 73.82$\pm$\tiny{2.50} & - &
        73.50$\pm$\tiny{1.80} & 75.16$\pm$\tiny{1.49} & - \\
        & \multirow{1}{*}{RGCN} & 
        73.30$\pm$\tiny{4.90} & 74.98$\pm$\tiny{4.50} & \bf 75.09$\pm$\tiny{4.03}& 
        69.16$\pm$\tiny{9.97} & 69.24$\pm$\tiny{10.0} & 78.47$\pm$\tiny{5.24}&
        50.29$\pm$\tiny{2.08} & 51.52$\pm$\tiny{4.37} & 71.71$\pm$\tiny{7.59}&
        52.75$\pm$\tiny{4.75} & 57.27$\pm$\tiny{9.49} & 74.04$\pm$\tiny{1.15}\\
        & \multirow{1}{*}{RGIN} & 
        68.75$\pm$\tiny{6.59} & 70.54$\pm$\tiny{5.03} & 74.20$\pm$\tiny{2.93}& 
        77.65$\pm$\tiny{4.62} & 78.15$\pm$\tiny{4.60} & 77.73$\pm$\tiny{4.42}&
        64.20$\pm$\tiny{2.85} & 64.52$\pm$\tiny{2.58} & \bf 75.43$\pm$\tiny{3.50}&
        66.11$\pm$\tiny{1.77} & 66.11$\pm$\tiny{1.69} & \bf 76.18$\pm$\tiny{2.03}\\
        \cline{2-14} & \multirow{1}{*}{DiffPool} & 
        75.62$\pm$\tiny{5.17} & \bf 75.98$\pm$\tiny{3.89} & - & 
        81.41$\pm$\tiny{5.11} & 80.25$\pm$\tiny{4.69} & - &
        75.29$\pm$\tiny{1.85} & \bf 75.44$\pm$\tiny{1.90} & - &
        76.62$\pm$\tiny{1.93} & \bf 77.08$\pm$\tiny{1.33} & - \\ 
        & \multirow{1}{*}{HGP-SL} & 
        71.25$\pm$\tiny{7.13} & 74.46$\pm$\tiny{3.77} & - & 
        74.62$\pm$\tiny{3.19} & \bf 82.07$\pm$\tiny{2.11} & - &
        74.78$\pm$\tiny{2.37} & 74.32$\pm$\tiny{1.84} & - & 
        74.94$\pm$\tiny{0.88} & 76.08$\pm$\tiny{1.94} & - \\
        \hline
        \hline
        \multicolumn{2}{c|}{\multirow{1}{*}{\it Average}} & 
        \it 73.13$\pm$\tiny{2.10} & \it 74.77$\pm$\tiny{2.60} & \it 75.31$\pm$\tiny{3.53} & 
        \it 77.20$\pm$\tiny{3.26} & \it 78.59$\pm$\tiny{3.05} & \it 79.31$\pm$\tiny{1.13} & 
        \it 72.07$\pm$\tiny{11.20} & \it 72.62$\pm$\tiny{10.53} & \it 77.72$\pm$\tiny{6.51} & 
        \it 73.48$\pm$\tiny{10.16} & \it 75.10$\pm$\tiny{8.98} & \it 78.27$\pm$\tiny{7.77} \\
        \bottomrule
    \end{tabular}
    \caption{Accuracies on graph classification, where ``OOM'' means out-of-memory, the best results are highlighted in bold, and the average results are italicized.}
    \label{table:graph_classification}
\end{table*}

\section{Experiment}

To evaluate the effectiveness of dummy nodes, we conduct two (sub)graph-level tasks.
On graph classification, we use both kernel functions and graph neural networks;
on subgraph isomorphism counting and matching, we only employ neural methods upon the generalization.
Appendix~\ref{appendix:exp} provides more details of experiments.

\subsection{Graph Classification}
\label{section:gc}

\noindent\textbf{Datasets.}
We select four benchmarking datasets where current state-of-the-art models face overfitting problems: PROTEINS~\cite{borgwardt2005protein}, D\&D~\cite{dobson2003distinguishing}, NCI109, and NCI1~\cite{wale2008comparison}.
The highest accuracies on these datasets are around $70\%\sim90\%$, and we want to explore the gain of dummy nodes.

\noindent\textbf{Graph Kernels.}
We use the open-source GKs for a fair comparison.
\citeauthor{morris2020weisfeiler}~(\citeyear{morris2020weisfeiler}) released a toolkit with an efficient C++ implementation of current best-performed GKs.\footnote{\url{https://www.github. com/chrsmrrs/sparsewl}}
Based on their experiments, we choose eight kernel functions: Shortest-path Kernel (SP)~\cite{borgward2005shortest}, Graphlet Kernel (GR)~\cite{shervashidze2009efficient}, Weisfeiler-Lehman Optimal Assignment Kernel (WLOA)~\cite{kriege2016on}, Weisfeiler-Lehman Subtree Kernels (1-WL and 2-WL)~\cite{shervashidze2011weisfeiler}, their proposed $\delta$-2-WL, $\delta$-2-LWL, and $\delta$-2-LWL$^+$.
The kernel functions for $\mathcal{G}_{\varphi}$ and $\mathcal{H}_{\Phi}$ are described in Sec.~\ref{section:extension_gk}.
Specifically, to handle $\mathcal{G}$, $\mathcal{G}_{\varphi}$ and $\mathcal{H}_{\Phi}$, GKs are equipped with their original kernel functions, Eq.~(\ref{eq:kernel_varphi}), and Eq.~(\ref{eq:kernel_Phi}), respectively.

\noindent\textbf{Graph Neural Networks.}
We adopt the PyG library~\cite{fey2019pyg}
to implement the neural baselines.
We consider the three most-famous networks, GraphSAGE~\cite{hamilton2017inductive}, GCN~\cite{kipf2017semi} and GIN~\cite{xu2019how}, as well as two state-of-the-art pooling-based networks, DiffPool~\cite{ying2018hierarchical} and HGP-SL~\cite{zhang2019hierarchical}.
Seeing that edge-to-vertex-transformed graphs ($\mathcal{H}_{\Phi}$) involve edge labels, we use two relational GNNs, RGCN~\cite{schlichtkrull2018modeling} and RGIN~\cite{liu2020neural}, to better utilize edge type information.
\revisexin{Note that we do not concatenate features from $\mathcal{G}$ for fairness, and we hope GNNs can learn from data.}

\begin{table*}[!ht]
    \small
    \centering
    \setlength\tabcolsep{4pt}
    \begin{tabular}{ll|ccc|ccc|ccc|ccc}
    \toprule
        \multicolumn{2}{c|}{\multirow{3}{*}{Models}} & \multicolumn{6}{c|}{Homogeneous}  & \multicolumn{6}{c}{Heterogeneous} \\
        \cline{3-14}
        & & \multicolumn{3}{c|}{Erd\H{o}s-Renyi} & \multicolumn{3}{c|}{Regular} & \multicolumn{3}{c|}{Complex} & \multicolumn{3}{c}{MUTAG} \\
        & & RMSE & MAE & GED & RMSE & MAE & GED & RMSE & MAE & GED & RMSE & MAE & GED \\
        \midrule
        \multirow{2}{*}{RGCN}
        & $\mathcal{G}$ & 9.386 & 5.829 & 28.963 & 14.789 & 9.772 & 70.746 & 28.601 & 9.386 & 64.122 & 0.777 & 0.334 & 1.441 \\
        & $\mathcal{G}_{\varphi}$ & 7.764 & 4.654 & 24.438 & 14.077 & 9.511 & 71.393 & 26.389 & 7.110 & 55.600 & 0.534 & 0.191 & 1.052 \\
        \hline
        \multirow{2}{*}{RGIN}
        & $\mathcal{G}$ & 6.063 & 3.712 & 22.155 & 13.554 & 8.580 & 56.353 & 20.893 & 4.411 & 56.263 & 0.273 & 0.082 & 0.329 \\
        & $\mathcal{G}_{\varphi}$ & \bf 4.769 & \bf 2.898 & \bf 15.219 & \bf 10.871 & \bf 6.874 & \bf 43.537 & 19.436 & \bf 3.846 & \bf 41.337 & \bf 0.193 & \bf 0.064 & 0.277 \\
        \hline
        \multirow{2}{*}{HGT}
        & $\mathcal{G}$ &24.376  & 14.630 & 104.000 & 26.713 & 17.482 & 191.674 & 34.055 & 8.336 & 70.080 & 1.317 & 0.526 & 3.644 \\
        & $\mathcal{G}_{\varphi}$ & 5.969 & 3.691 & 23.401 & 13.813 & 8.813 & 64.926 & 20.841 & 4.707 & 47.409 & 0.876 & 0.345 & 2.973 \\
        \hline
        \multirow{2}{*}{CompGCN}
        & $\mathcal{G}$ & 6.706 & 4.274 & 25.548 & 14.174 & 9.685 & 64.677 & 22.287 & 5.127 & 57.082 & 0.300 & 0.085 & 0.278 \\
        & $\mathcal{G}_{\varphi}$ & 4.981 & 3.019 & 16.263 & 11.450 & 7.443 & 46.802 & 20.786 & 4.048 & 56.269 & 0.321 & 0.089 & \bf 0.262 \\
        \hline
        \multirow{2}{*}{DMPNN}
        & $\mathcal{G}$ & 5.330 & 3.308 & 23.411 & 11.980 & 7.832 & 56.222 & 18.974 & 3.992 & 56.933 & 0.232 & 0.088 & 0.320 \\
        & $\mathcal{G}_{\varphi}$ & 5.220 & 3.130 & 23.285 & 11.259 & 7.136 & 49.179 & \bf 18.885 & 3.892 & 73.161 & 0.259 & 0.101 & 0.623 \\
        \hline
        \hline
        \multirow{2}{*}{Deep-LRP}
        & $\mathcal{G}$ & 0.794 & 0.436 & 2.571 & 1.373 & 0.788 & 5.432 & 27.490 & 5.850 & 56.772 & 0.260 & 0.094 & 0.437 \\
        & $\mathcal{G}_{\varphi}$ & 0.710 & 0.402 & 2.218 & 1.145 & 0.718 & 4.611 & 24.458 & 5.094 & 57.398 & 0.356 & 0.115 & 0.849 \\
        \hline
        \multirow{2}{*}{DMPNN-LRP}
        & $\mathcal{G}$ & \bf 0.475 & 0.287 & 1.538 & \bf 0.617 & 0.422 & 2.745 & 20.425 & 4.173 & \bf 32.200 & 0.196 & 0.062 & \bf 0.210 \\
        & $\mathcal{G}_{\varphi}$ & 0.477 & \bf 0.260 & \bf 1.457 & 0.633 & \bf 0.413 & \bf 2.538 & \bf 18.127 & \bf 4.112 & 39.594 & \bf 0.186 & \bf 0.057 & 0.265 \\
        \bottomrule
    \end{tabular}
    \caption{Performance on subgraph isomorphism counting and matching.}
    \label{table:subisocnt}
\end{table*}


\noindent\textbf{Results and Discussion.}
Experimental results on graph classification benchmarks are shown in Table~\ref{table:graph_classification}.
Note that some models (GCN, GIN, GraphSAGE, DiffPool, and HGP-SL) are not designed to handle edge types. For a fair comparison, we do not evaluate these models' performance on the transformed graph $\mathcal{H}_{\Phi}$. 
We observe consistent improvement in performance after adding dummy nodes ($\mathcal{G}_{\varphi}$) to the original graphs ($\mathcal{G}$) for most classifiers.
When the input is the graphs ($\mathcal{H}_{\Phi}$) transformed by $L_{\Phi}$, we also see the further improvement on average.
Any progress on graph kernels is not easy, but our kernel modification helps almost all kernels over four datasets.
And kernels with the 2-order structures nearly surpass the 1-order kernels and classical graph neural networks with the expressive power no more than 1-WL.
Obtaining global high-order structure statistics is time-consuming, but the local variants  $\delta$-2-LWL, and $\delta$-2-LWL$^+$ achieves the best tradeoff between efficiency and effectiveness.\footnote{
Even for the fastest $\delta$-3-LWL$^+$, it either requires over 20,000 seconds (for NCI109 and NCI1) or causes the out-of-memory issue (for PROTEINS and D\&D), but there is no further improvement.}
GNNs with pooling capture the hierarchical graph structures and outperform simple GNNs.
GNNs with $\mathcal{G}_{\varphi}$ can compete against DiffPool with $\mathcal{G}_{\varphi}$, but DiffPool and HGP-SL can further benefit from the artificial dummy nodes.
On the other hand, relational GNNs can handle various edge types in $\mathcal{H}_{\Phi}$.
Accuracies get boosted again when the input changes from $\mathcal{G}_{\varphi}$ to $\mathcal{H}_{\Phi}$.
\revisexin{
We also notice the unstable performance in GCN and RGCN on D\&D and NCI1.
Using $\mathcal{G}_{\varphi}$ instead of $\mathcal{G}$ cannot solve their oversmoothing problem, although their performance becomes slightly better to a certain degree.
As a comparison, GraphSAGE using $Max$ and GIN using $Sum$ can provide the steady criterion whenever the input graphs with dummy or not.
One negative observation is the over-parameterization problem in RGCN and RGIN.
They can perform as expected only when feeding $\mathcal{H}_{\Phi}$ with numerous edge types.
How to help relational GNNs handle dummy edges is one of our future work.
We also think the combination of pooling with heterogeneous information is promising.
}

\noindent\textbf{Relevance to Over-smoothing.}
\revisexin{Adding dummy nodes to graphs can help alleviate the over-smoothing issue.
For instance, on the NCI1 dataset, 
a 2-layer GIN
achieves 73.50±1.80 (75.16±1.49) accuracy on $\mathcal{G}$ ($\mathcal{G}_{\varphi}$);
a 4-layer model
achieves 71.97±1.46 (75.11±1.88) accuracy on $\mathcal{G}$ ($\mathcal{G}_{\varphi}$).
We can see that: (1) on $\mathcal{G}$, the model performs worse while its number of layers increases, indicating the over-smoothing does exist, and (2) 
on $\mathcal{G}_{\varphi}$,
the model achieves comparable performance after stacking layers, meaning adding dummy nodes help overcome the over-smoothing.
}

\revisexin{
To understand why, 
let's take a look at a 2-layer GNN.
The first layer helps every vertex receive one-hop information.
At this time, the dummy node aggregates all vertex information.
The second layer helps every vertex receive two-hop information (from its neighbors) and global information (from the dummy node).
The dummy provides every vertex $v$ with additional information of all vertices, helping $v$ differentiate the intersection of the one-hop and the two-hop, and other vertices beyond two hops.
In this way, with the help of the ``shortcut'' by the dummy, the learned vertex representations of are more distinct from each other.
This also applies for deeper layers, and we observe the models with dummy nodes are more robust against over-smoothing.
}

\subsection{Subgraph Isomorphism Counting and Matching}
\label{section:sic}

\noindent\textbf{Datasets.}
We evaluate neural methods on two synthetic homogeneous datasets with 4 graphlet patterns~\cite{chen2020can}, one synthetic heterogeneous dataset with 75 random patterns up to eight vertices, and one mutagenic compound dataset MUTAG with 24 artificial patterns~\cite{liu2020neural}.

\noindent\textbf{End-to-end Framework}
Considering the complexity and the generalization, we follow the neural framework proposed by \citeauthor{liu2020neural}~(\citeyear{liu2020neural}) and employ their released implementation.\footnote{\url{https://github.com/HKUST-KnowComp/DualMessagePassing}}
The end-to-end framework includes four parts: encoding, representation, fusion, and prediction.
It supports sequence models (e.g., RNN) and graph models (e.g., RGCN).
Based on their practical experience, we only consider the effective graph models, including RGCN, RGIN, CompGCN, DMPNN, Deep-LRP, and DMPNN-LRP~\cite{liu2021graph}.
\revisexin{We also implement HGT~\cite{hu2020heterogeneous} to see whether relation-specific attention performs well and whether dummy nodes can help it.}
\revisexin{
The framework utilizes one feed-forward network to make predictions at the graph level, i.e., $\text{FFN}_{\text{counting}}(Concat(\bm x_{v}, \bm p, \bm x_{v} - \bm p, \bm x_{\mathcal{G}_{v}} \odot \bm p))$, and another one feed-forward network to make predictions at the vertex level, i.e., $\text{FFN}_{\text{matching}}(Concat(\bm g, \bm p, \bm g - \bm p, \bm g \odot \bm p))$, where $\bm p$ is the pattern representation by applying pooling over the pattern's vertex representations, $\bm x_{v}$ is the graph vertex representation, and $\bm g$ is the sum of $\{\bm x_{v} | v \in \mathcal{V}_{\mathcal{G}}\}$.
}

\noindent\textbf{Evaluation Metric.}
The counting prediction is modeled as regression, so 
\revisexin{RMSE and MAE}
are adopted to estimate global inference skills.
\revisexin{
To evaluate the ability of local decision making, we require models to predict the frequency of each node that how many times it appears in all isomorphic subgraphs, 
\revisexin{so graph edit distance (GED) serves as a metric.}
}

\begin{figure}[t]
    \centering
    \includegraphics[width=\linewidth]{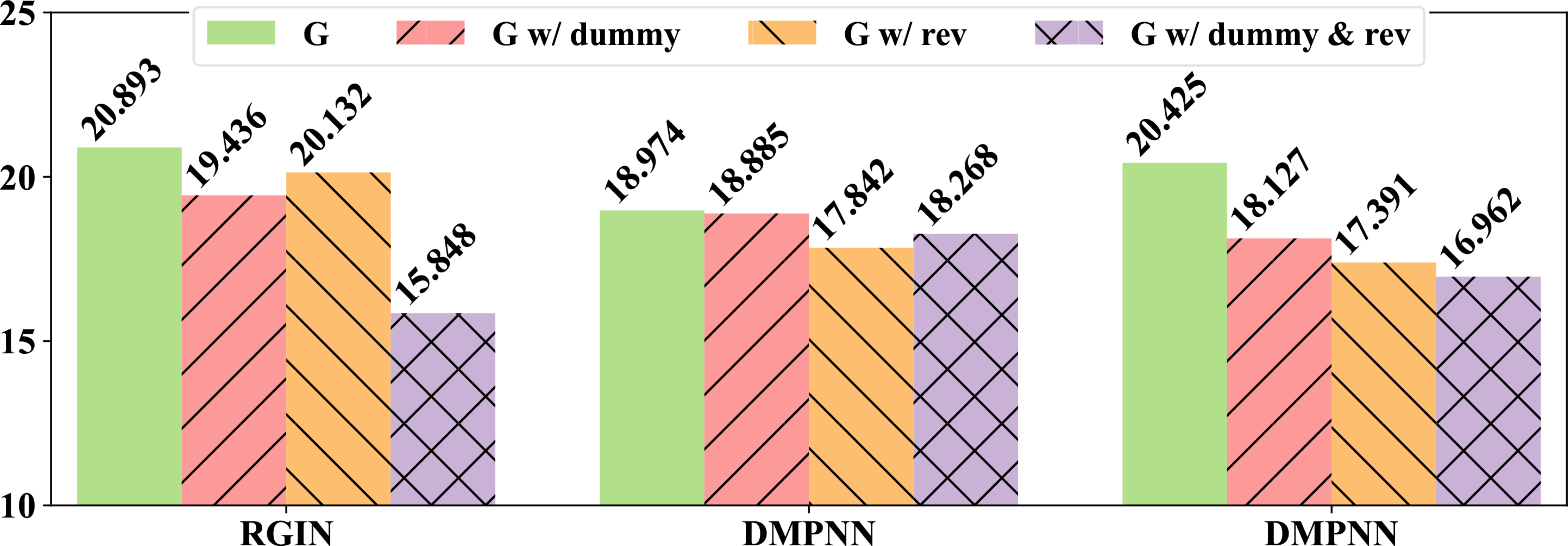}
    \vspace{-0.3in}
    \caption{RMSE of RGIN, DMPNN, and DMPNN-LRP with different input graph structures on the \textit{Complex} dataset.}
    \label{fig:complex}
    \vspace{-0.15in}
\end{figure}

\noindent\textbf{Results and Discussion.}
Table~\ref{table:subisocnt} lists the performance on subgraph isomorphism counting and matching.
Almost all graph models consistently benefit from the dummy nodes.
In particular, RGIN outperforms other GNNs.
Both RGCN and RGIN use relation-specific matrices to transform neighbor messages, but we still observe that RGIN has greater average relative error reductions ($19.35\%$ of RMSE and $24.06\%$ of GED) than RGCN ($15.28\%$ of RMSE and $13.75\%$ of GED).
One interesting observation is that HGT has the most prominent performance boost by a $49.02\%$ relative error reduction of RMSE, a $50.58\%$ relative error reduction of MAE, and a $49.59\%$ relative error reduction of GED on average.
We can explain that the dummy nodes provide an option to drop all pattern-irrelevant messages.
Without such dummy nodes and edges, irrelevant messages would always be aggregated as side-effects.
DMPNN, as the previous state-of-the-art model, has minor overall improvement and even a slight drop of GED on heterogeneous data.
\revisexin{The possible reason is the star topology, making the edge representation learning based on line graphs difficult.}
CompGCN fixes edge presentations and focuses on graph structures, yielding a $17.79\%$ reduction of GED.
On the other hand, DMPNN-LRP combines dual message passing with pooling on explicit neighbor permutations and obtains the lowest errors on homogeneous data, but it still faces the same problem that matching errors increase on heterogeneous graphs.

We also consider adding reversed edges and dummy nodes and edges simultaneously.
Figure~\ref{fig:complex} illustrates the counting error changes on the \textit{Complex} data.
We note that the two strategies work together well on RGIN, but the boost of DMPNN and DMPNN-LRP mainly comes from reversed edges.
We believe this issue can be addressed in the future since we see their cooperation in DMPNN-LRP and the success of graph classification with $\mathcal{H}_{\Phi}$ in Sec.~\ref{section:gc}.
\section{Expressive Power of MPNNs with Dummy Nodes and Transformers with CLS Tokens}
\revisexin{
In fact, the edge-to-vertex transform $L_{\Phi}$ corresponds to the construction of local 2-tuples in $\delta$-2-LWL$^+$~\cite{morris2020weisfeiler}.
And we can conclude that message passing neural networks (MPNNs) with $\mathcal{H}_{\Phi}$ have the same expressive power of $\delta$-2-LWL$^+$ with $\mathcal{G}$.
And it has been proven that $\delta$-$k$-LWL$^+$ is strictly more powerful than $k$-WL.
That is, MPNNs with $\mathcal{H}_{\Phi}$ are more powerful than $2$-WL with $\mathcal{G}$.
And we also know that MPNNs are no more powerful than $2$-WL~\cite{chen2020can} and usually as same powerful as 1-WL~\cite{xu2019how}, so we have successfully empowered the MPNNs to surpass 2-WL with the help of $\mathcal{H}_{\Phi}$.
The inverse transform $L^{-1}_{\Phi}$ can always recover $\mathcal{G}$ back if the vertex id information is provided.
This implies that MPNNs with $\mathcal{G}_{\varphi}$ and vertex id information should have the same expressive power as MPNNs with $\mathcal{H}_{\Phi}$.
}

\revisexin{
Graph attention networks (GATs)~\cite{velickovic2018graph} and heterogeneous graph transformers (HGTs)~\cite{hu2020heterogeneous} also belong to the message passing framework.
Therefore, the above discussions are applicable to them.
Transformer-based encoders~\cite{vaswani2017attention} regard the input as a fully-connected graph and adopt attention to explicitly learn pair-wise connections and implicitly learn global patterns.
For each vertex (token), all other vertices and itself are served as its neighbors, indicating a stronger discriminative capability than GATs and HGTs.
Thus, transformers with dummy nodes (usually named CLS tokens) should be more powerful than 2-WL.
Based on the analyses in~\cite{chen2020can}, a 12-layer transformer encoder can capture patterns size of at most $3 \cdot 2^{12}=12288$, and a 24-layer model can almost learn any patterns.
And that is the reason why transformers dominate the sequence encoding.
}
\section{Conclusion}
In this paper, we analyze the role of dummy nodes in the lossless edge-to-vertex transform.
We further prove that a dummy node with connections to all existing vertices can preserve the graph structure.
Specifically, we design an efficient monomorphic edge-to-vertex transform and find its inverse to recover the original graph back.
We extend graph kernels and graph neural networks with dummy nodes.
Experiments demonstrate the success of performance boost on graph classification and subgraph isomorphism counting and matching.
\revisexin{Last, we discuss the capability of MPNNs and Transformers with special dummy elements.}

\section*{Acknowledgements}
The authors of this paper were supported by the NSFC Fund (U20B2053) from the NSFC of China, the RIF (R6020-19 and R6021-20) and the GRF (16211520) from RGC of Hong Kong, the MHKJFS (MHP/001/19) from ITC of Hong Kong and the National Key R\&D Program of China (2019YFE0198200) with special thanks to HKMAAC and CUSBLT, and  the Jiangsu Province Science and Technology Collaboration Fund (BZ2021065). We also thank the support from the UGC Research Matching Grants (RMGS20EG01-D, RMGS20CR11, RMGS20CR12, RMGS20EG19, RMGS20EG21).
\clearpage

\bibliography{icml2022}
\bibliographystyle{icml2022}

\clearpage

\appendix

\section{Algorithms for $L_{\Phi}$ and $L_{\Phi}^{-1}$}
\label{appendix:alg}

\begin{algorithm}[!h]
    \caption{Edge-to-vertex transform $L_{\Phi}$}
    \label{alg:evt}
    {\footnotesize
        \begin{algorithmic}[1]
        \INPUT a connected directed graph $\mathcal{G}$ with $n$ vertices and $m$ edges ($m > 0$), a special vertex label for dummy nodes $x_{\varphi}$
        \STATE let $y_{\varphi}$ as the special edge label for dummy edges
        \STATE add one dummy node $\varphi$ with label $x_{\varphi}$
        \FOR{each $v$ in $\mathcal{V}_{\mathcal{G}} - \{\varphi\}$}
            \STATE add two dummy edges $(v, \varphi)$ and $(\varphi, v)$ with label $y_{\varphi}$
        \ENDFOR \hfill  // \COMMENT{end of step \textcircled{\raisebox{-.9pt} {1}}}
        \STATE get the line graph $\mathcal{H}_{\Phi} = L(\mathcal{G})$
        \STATE assign edge ids with original vertex ids \hfill // \COMMENT{end of step \textcircled{\raisebox{-.9pt} {2}}}
        \STATE add a single dummy node $\Phi$
        \FOR{each $e=(u, v)$ in $\mathcal{E}_{\mathcal{H}_{\Phi}}$}
            \IF{$u$ is associated with label $y_{\varphi}$ \AND $v$ is associated with label $y_{\varphi}$}
            \STATE delete $e$
            \ELSIF{$u$ is associated with label $y_{\varphi}$}
            \STATE add one edge $(\Phi, v)$ with $e$'s id and labels $\mathcal{Y}_{\mathcal{H}_{\Phi}}(e)$
            \STATE delete $e$
            \ELSIF{$v$ is associated with label $y_{\varphi}$}
            \STATE add one edge $(u, \Phi)$ with $e$'s id and labels $\mathcal{Y}_{\mathcal{H}_{\Phi}}(e)$
            \STATE delete $e$
            \ENDIF
        \ENDFOR
        \FOR{each $v$ in $\mathcal{V}_{\mathcal{H}_{\Phi}}$}
            \IF{$v$ is associated with label $y_{\varphi}$}
                \STATE delete $v$
            \ENDIF
        \ENDFOR \hfill // \COMMENT{end of step \textcircled{\raisebox{-.9pt} {3}}}
        \OUTPUT the transformed graph $\mathcal{H}_{\Phi}$
        \end{algorithmic}
    }
\end{algorithm}

\begin{algorithm}[h!]
    \caption{Inverse edge-to-vertex transform $L_{\Phi}^{-1}$}
    \label{alg:ievt}
    {\footnotesize
        \begin{algorithmic}[1]
        \INPUT a transformed graph $\mathcal{H}_{\Phi}$ obtained by $L_{\Phi}$, a special vertex label for dummy nodes $x_{\varphi}$
        \STATE get the line graph $\mathcal{G}=L(\mathcal{H}_{\Phi})$
        \STATE assign vertex ids with the original edge ids
        \hfill // \COMMENT{end of step \textcircled{\raisebox{-.9pt} {1}}}
        \STATE create an empty mapping $\mathcal{I}$
        \FOR{each $e=(u, v)$ in $\mathcal{E}_{\mathcal{G}}$}
            \IF{$u.id$ not in $\mathcal{I}$ \AND $v.id$ not in $\mathcal{I}$}
            \STATE set $\mathcal{I}(u.id) \leftarrow u$ and $\mathcal{I}(v.id) \leftarrow v$
            \ELSIF{$u.id$ not in $\mathcal{I}$}
            \STATE add one edge $(u, \mathcal{I}(v.id))$ with $e$'s id and  labels
            \STATE delete $e$
            \STATE set $\mathcal{I}(u.id) \leftarrow u$
            \ELSIF{$v.id$ not in $\mathcal{I}$}
            \STATE add one edge $(\mathcal{I}(u.id), v)$ with $e$'s id and labels
            \STATE delete $e$
            \STATE set $\mathcal{I}(v.id) \leftarrow v$
            \ELSIF{$u \neq \mathcal{I}(u.id)$ \OR $v \neq \mathcal{I}(v.id)$}
            \STATE add one edge $(\mathcal{I}(u.id), \mathcal{I}(v.id))$ with $e$'s id and labels
            \STATE delete $e$
            \ENDIF
        \ENDFOR
        \FOR{each $v$ in $\mathcal{V}_{\mathcal{G}}$}
            \IF{$v \neq \mathcal{I}(v.id)$}
                \STATE delete $v$
            \ENDIF
        \ENDFOR \hfill // \COMMENT{end of step \textcircled{\raisebox{-.9pt} {2}}}
        \FOR{each $e$ in $\mathcal{E}_{\mathcal{G}}$}
            \IF{$e$ is associated with label $x_{\varphi}$}
                \STATE delete $e$
            \ENDIF
        \ENDFOR \hfill // \COMMENT{end of step \textcircled{\raisebox{-.9pt} {3}}}
        \OUTPUT the transformed graph $\mathcal{G}$
        \end{algorithmic}
    }
\end{algorithm}

\section{Details of Experiments}
\label{appendix:exp}
\subsection{Environment}
We conduct our experiments on one CentOS 7 server with 2 Intel Xeon Gold 5215 CPUs and 4 NVIDIA GeForce RTX 3090 GPUs.
The software versions are: GNU C++ Compiler 5.2.0, Python 3.7.3, PyTorch 1.7.1, torch-geometric 2.0.2, and DGL 0.6.0.


\subsection{Graph Classification}
\label{appendix:exp_gc}

\subsubsection{Datasets}
Dataset statistics are listed in Table~\ref{table:stat_gc}. 
To construct the dataset with dummy node information, we add an extra dummy node to each graph and connect it with all the other vertices in the  graph with bidirectional edges. 
\revisexin{When one graph is undirected, we employ the common practice to replace one undirected edge with one directed edge and its reverse.}
Following previous work~\cite{zhang2019hierarchical}, we randomly split each dataset into the training set (80\%), the validation set (10\%), and the test set (10\%) in each run.

\subsubsection{Implementation Details}
\begin{itemize}
\item {\noindent\textbf{Kernel Methods.}
We compile kernel functions with C++11 features and -O2 flag.
After obtaining normalized Gram matrices, SVM classifiers are trained based on LibSVM\footnote{\url{https://www.csie.ntu.edu.tw/~cjlin/libsvm}} and wrapped by sklearn.
}
\item {\noindent\textbf{Graph Neural Network Based Methods.}
We implement and evaluate all our graph neural network based models with the PyG library. 
For GraphSAGE, GIN and DiffPool, we adapt the implementation by \citeauthor{errica2020afair}~(\citeyear{errica2020afair}).\footnote{\url{https://github.com/diningphil/gnn-comparison}}
To implement our RGCN, we modify the PyG implementation\footnote{\url{https://github.com/pyg-team/pytorch_geometric/blob/master/examples/rgcn.py}} by adding 3 fully connected layers after the convolutional layers.
The RGIN convolutional layer can be adapted from a RGCN convolutional layer, by setting the aggregation function as \textit{Sum}, and followed by a multi-layer perceptron.
For HGP-SL, we use the official code.\footnote{\url{https://github.com/cszhangzhen/HGP-SL}}
We observe a performance drop compared with the results reported in their paper, either when it runs with a newer version of torch-sparse (in our setting, torch-sparse=0.6.9), or with the version reported in their GitHub page (torch-sparse=0.4.0). 
For a fair comparison with other baseline models, we choose to use our current software versions.
}
\end{itemize}

\subsubsection{Hyper-parameter Settings}
For reproducibility, we run all the experiments $10$ times with random seeds $\{2020, 2021, \cdots, 2029\}$, and report the sample mean and standard deviation of test accuracies. 

\begin{itemize}
\item {\noindent\textbf{Kernel Methods.} 
For kernel methods, we search for the regularization parameter $C$ within $\{10^{-7}, 10^{-6}, \cdots, 10^{3}\}$ for each seed and corresponding training data.
The best hyper-parameter setting is chosen by the best validation performance.
}
\item {\noindent\textbf{Graph Neural Network Based Methods.}
We use the Adam optimizer~\cite{kingma2015adam} to optimize the models. 
Following \citeauthor{zhang2019hierarchical}~(\citeyear{zhang2019hierarchical}), an early stopping strategy with patience 100 is adopted during training, i.e., training would be stopped when the loss on validation set does not decrease for over 100 epochs. 
For GraphSAGE, GIN and DiffPool, the optimal hyper-parameters are found using grid search within the same search ranges in \cite{errica2020afair}. 
For HGP-SL, we follow the official hyper-parameters reported in their GitHub repository.
\revisejy{For models using graph convolutional operator \cite{kipf2017semi} (GCN, DiffPool, HGP-SL), we additionally impose a learnable weight $\gamma$ on the dummy edges.
The weights for all the other edges is set as 1, and $\gamma$ is initialized with different values in $\{0.01, 0.1, 1, 10\}$.}
For relational models RGCN and RGIN, we search for the learning rate within $\{1e-2, 1e-3, 1e-4\}$, batch size $\in \{128, 512\}$, hidden dimension $\in \{32, 64\}$, dropout ratio $\in \{0, 0.5\}$, and number of layers $\in \{2, 4\}$.
}
\end{itemize}

\begin{table}[!t]
    \footnotesize
    \centering
    \setlength\tabcolsep{4pt}
    \resizebox{1.\linewidth}{!}{%
    \begin{tabular}{cl|c|c|c|c|c|c}
    \toprule
        \multicolumn{2}{c|}{Dataset} & \# Graphs & \# Classes & Avg. $|\mathcal{V}_{\mathcal{G}}|$ & Avg. $|\mathcal{E}_{\mathcal{G}}|$ & $|\mathcal{X}_{\mathcal{G}}|$ & $|\mathcal{Y}_\mathcal{G}|$\\
        \toprule
        \multirow{3}{*}{PROTEINS}
        & $\mathcal{G}$             & 1,113 & 2 & 39.1 & 145.6 & 3 & 1\\
        & $\mathcal{G}_{\varphi}$   & 1,113 & 2 & 40.1 & 223.7 & 4 & 2\\
        & $\mathcal{H}_{\Phi}$      & 1,113 & 2 & 146.6 & 885.9 & 2 & 4\\
        \midrule
        \multirow{3}{*}{D\&D}
        & $\mathcal{G}$             & 1,178 & 2 & 284.3 & 1431.3 & 82 & 1\\
        & $\mathcal{G}_{\varphi}$   & 1,178 & 2 & 285.3 & 2000.0 & 83 & 2\\
        & $\mathcal{H}_{\Phi}$      & 1,178 & 2 & 1432.3 & 10875.5 & 2 & 83\\
        \midrule
        \multirow{3}{*}{NCI109}
        & $\mathcal{G}$             & 4,127 & 2 & 29.7 & 64.3 & 38 & 1\\
        & $\mathcal{G}_{\varphi}$   & 4,127 & 2 & 30.7 & 123.6 & 39 & 2\\
        & $\mathcal{H}_{\Phi}$      & 4,127 & 2 & 65.3 & 285.8 & 2 & 39\\
        \midrule
        \multirow{3}{*}{NCI1}
        & $\mathcal{G}$             & 4,110 & 2 & 29.9 & 64.6 & 37 & 1\\
        & $\mathcal{G}_{\varphi}$   & 4,110 & 2 & 30.9 & 124.3 & 38 & 2\\
        & $\mathcal{H}_{\Phi}$      & 4,110 & 2 & 65.6 & 287.0 & 2 & 38\\
        \bottomrule
    \end{tabular}
    }
    \caption{Dataset statistics on graph classification.}
    \label{table:stat_gc}
\end{table}

\subsection{Subgraph Isomorphism Counting and Matching}
\label{appendix:exp_sic}

\begin{table}[!t]
    \footnotesize
    \centering
    \setlength\tabcolsep{1.2pt}
    \resizebox{\linewidth}{!}{%
    \begin{tabular}{l|ccc|ccc|ccc|ccc}
    \toprule
        & \multicolumn{3}{c|}{Erd\H{o}s-Renyi} & \multicolumn{3}{c|}{Regular} & \multicolumn{3}{c|}{Complex} & \multicolumn{3}{c}{MUTAG} \\
        \toprule
        \# Train & \multicolumn{3}{c|}{6,000} & \multicolumn{3}{c|}{6,000} & \multicolumn{3}{c|}{358,512} & \multicolumn{3}{c}{1,488} \\
        \# Valid & \multicolumn{3}{c|}{4,000} & \multicolumn{3}{c|}{4,000} & \multicolumn{3}{c|}{44,814} & \multicolumn{3}{c}{1,512} \\
        \# Test & \multicolumn{3}{c|}{10,000} & \multicolumn{3}{c|}{10,000} & \multicolumn{3}{c|}{44,814} & \multicolumn{3}{c}{1,512} \\
        \midrule
        & Max & \multicolumn{2}{c|}{Avg.} & Max & \multicolumn{2}{c|}{Avg.} & Max & \multicolumn{2}{c|}{Avg.} & Max & \multicolumn{2}{c}{Avg.} \\
        $|\mathcal{V}_{\mathcal{P}}|$ & 4 & \multicolumn{2}{c|}{3.8$\pm$0.4} & 4 & \multicolumn{2}{c|}{3.8$\pm$0.4} & 8 & \multicolumn{2}{c|}{5.2$\pm$2.1} & 4 & \multicolumn{2}{c}{3.5$\pm$0.5} \\
        $|\mathcal{E}_{\mathcal{P}}|$ & 10 & \multicolumn{2}{c|}{7.5$\pm$1.7} & 10 & \multicolumn{2}{c|}{7.5$\pm$1.7} & 8 & \multicolumn{2}{c|}{5.9$\pm$2.0} & 3 & \multicolumn{2}{c}{2.5$\pm$0.5} \\
        $|\mathcal{X}_{\mathcal{P}}|$ & 1 & \multicolumn{2}{c|}{1$\pm$0} & 1 & \multicolumn{2}{c|}{1$\pm$0} & 8 & \multicolumn{2}{c|}{3.4$\pm$1.9} & 2 & \multicolumn{2}{c}{1.5$\pm$0.5} \\
        $|\mathcal{Y}_{\mathcal{P}}|$ & 1 & \multicolumn{2}{c|}{1$\pm$0} & 1 & \multicolumn{2}{c|}{1$\pm$0} & 8 & \multicolumn{2}{c|}{3.8$\pm$2.0} & 2 & \multicolumn{2}{c}{1.5$\pm$0.5} \\
        $|\mathcal{V}_{\mathcal{G}}|$ & 10 & \multicolumn{2}{c|}{10$\pm$0} & 30 & \multicolumn{2}{c|}{18.8$\pm$7.4} & 64 & \multicolumn{2}{c|}{32.6$\pm$21.2} & 28 & \multicolumn{2}{c}{17.9$\pm$4.6} \\
        $|\mathcal{E}_{\mathcal{G}}|$ & 48 & \multicolumn{2}{c|}{27.0$\pm$6.1} & 90 & \multicolumn{2}{c|}{62.7$\pm$17.9} & 256 & \multicolumn{2}{c|}{73.6$\pm$66.8} & 66 & \multicolumn{2}{c}{39.6$\pm$11.4} \\
        $|\mathcal{X}_{\mathcal{G}}|$ & 1 & \multicolumn{2}{c|}{1$\pm$0} & 1 & \multicolumn{2}{c|}{1$\pm$0} & 16 & \multicolumn{2}{c|}{9.0$\pm$4.8} & 7 & \multicolumn{2}{c}{3.3$\pm$0.8} \\
        $|\mathcal{Y}_{\mathcal{G}}|$ & 1 & \multicolumn{2}{c|}{1$\pm$0} & 1 & \multicolumn{2}{c|}{1$\pm$0} & 16 & \multicolumn{2}{c|}{9.4$\pm$4.7} & 4 & \multicolumn{2}{c}{3.0$\pm$0.1} \\
        \bottomrule
    \end{tabular}
    }
    \caption{Dataset statistics on subgraph isomorphism experiments. $\mathcal{P}$ and $\mathcal{G}$ corresponds to patterns and graphs.}
    \label{table:stat_sic}
\end{table}



\subsubsection{Datasets}
The statistics of datasets on the subgraph isomorphism counting and matching task are listed in Table~\ref{table:stat_sic}.

\subsubsection{Implementation Details}
We adapt the DGL implementation provided by \citeauthor{liu2021graph}~(\citeyear{liu2021graph})\footnote{\url{https://github.com/HKUST-KnowComp/DualMessagePassing}} to evaluate the effect of dummy nodes on neural subgraph isomorphism counting and matching.
Apart from that, we add a new implementation of HGT~\cite{hu2020heterogeneous}.
\revisexin{
We learn from experience and lessons to jointly train models for counting and matching in a multitask setting.}

\subsubsection{Hyper-parameter Settings}
We follow the same paradigm for the training and evaluation in  the multitask learning setting. 
The best model over validation data among three random seeds $\{0, 2020, 2022\}$ is reported.

The embedding dimensions and hidden sizes are set as 64 for all 3-layer networks.
Deep-LRP and DMPNN-LRP enumerate neighbor subsets by 3-truncated BFS.
The HGT model is also set as the same hyper-parameters.
Residual connections and Leaky ReLU are added between two layers.
We use the AdamW optimizer~\cite{loshchilov2019decoupled} to optimize the models with a learning rate $1e-3$ and a weight decay $1e-5$.

\end{document}